\documentclass[opre,nonblindrev]{informs3aa} 
\OneAndAHalfSpacedXI

\usepackage{endnotes}
\let\footnote=\endnote
\let\enotesize=\normalsize
\def\notesname{Endnotes}%
\def\makeenmark{$^{\theenmark}$}
\def\enoteformat{\rightskip0pt\leftskip0pt\parindent=1.75em
  \leavevmode\llap{\theenmark.\enskip}}

\usepackage{natbib}
 \bibpunct[, ]{(}{)}{,}{a}{}{,}%
 \def\bibfont{\small}\def\bibsep{\smallskipamount}\def\bibhang{24pt}%

\usepackage[colorlinks=true,breaklinks=true,bookmarks=false,urlcolor=blue,
     citecolor=blue,linkcolor=blue,bookmarksopen=false,draft=false]{hyperref}
\usepackage[linesnumbered,ruled,vlined]{algorithm2e}
\SetKw{KwDownto}{downto}

\def\EMAIL#1{\href{mailto:#1}{#1}}

\usepackage{booktabs,multirow,graphics,latexsym,amsfonts,epsfig,verbatim,amsmath,color}
\usepackage{tabularx, booktabs}
\usepackage{mathtools}
\usepackage{subcaption}

\usepackage{color-edits}
\addauthor{hj}{blue}

\TheoremsNumberedThrough
\EquationsNumberedThrough

\begin{document}

\TITLE{Multi-User Contextual Cascading Bandits for Personalized Recommendation}
\ARTICLEAUTHORS{
\AUTHOR{Jiho Park, Huiwen Jia}
\AFF{Department of Industrial Engineering and Operations Research, University of California, Berkeley\\ 
\EMAIL{jihopark7959@berkeley.edu, huiwenj@berkeley.edu}}
}

\RUNAUTHOR{Park and Jia}
\RUNTITLE{Multi-User Contextual Cascading Bandits for Personalized Recommendation}

\ABSTRACT{We introduce a Multi-User Contextual Cascading Bandit model, a new combinatorial bandit framework that captures realistic online advertising scenarios where multiple users interact with sequentially displayed items simultaneously. Unlike classical contextual bandits, MCCB integrates three key structural elements: (i) cascading feedback based on sequential arm exposure, (ii) parallel context sessions enabling selective exploration, and (iii) heterogeneous arm-level rewards. 
We first propose Upper Confidence Bound with Backward Planning (UCBBP), a UCB-style algorithm tailored to this setting, and prove that it achieves a regret bound of $\widetilde{\mathcal{O}}(\sqrt{THN})$ over $T$ episodes, $H$ session steps, and $N$ contexts per episode. Motivated by the fact that many users interact with the system simultaneously, we introduce a second algorithm, termed Active Upper Confidence Bound with Backward Planning (AUCBBP), which shows a strict efficiency improvement in context scaling, i.e., user scaling, with a regret bound of $\widetilde{\mathcal{O}}(\sqrt{T+HN})$.
We validate our theoretical findings via numerical experiments, demonstrating the empirical effectiveness of both algorithms under various settings.}

\maketitle

\section{Introduction}
Bandit algorithms are widely used to solve online decision-making problems such as personalized recommendation~\cite{li2010a} and digital advertising~\cite{li2010b}. In these applications, each user interaction provides user-specific behavioral signals, and the platform must adaptively select an action, such as recommending a product or displaying an advertisement (ad), to maximize expected reward under uncertainty. To better capture user-specific preferences, \emph{contextual bandits} extend the classical multi-armed bandit (MAB) framework by incorporating side information, i.e., contextual features, into the decision process. While classical contextual bandits provide a powerful abstraction, they are not well-suited to scenarios where user behavior is inherently sequential. In practical settings, users are often presented with a list of items and examine them one by one, rather than choosing a single option. This has motivated the development of structured bandit models that capture ranked feedback. A representative example is the \emph{cascading bandit} \cite{pmlr-v37-kveton15}, where users sequentially view a ranked list and click on the first item they find attractive. A common assumption in prior works \cite{pmlr-v37-kveton15, 10.5555/3666122.3667137} is that user sessions terminate upon the first click, and the reward is modeled as binary.

However, real-world advertising systems present two key considerations that go beyond the classical setup. First, \emph{clicks yield heterogeneous revenue}, as the value of a click depends on which ad was clicked, not merely on whether a click occurred. Second, these systems often involve \emph{multiple users interacting with the platform simultaneously}, offering an opportunity to exploit economies of scale in learning and decision-making. Yet, this parallelism among multiple users remains underdeveloped in existing methods. Together, these two considerations motivate a more expressive modeling framework that transcends traditional assumptions.  Accordingly, we identify three structural components that are critical for accurately capturing the dynamics of modern advertising systems.
\begin{enumerate}
  \item \textbf{Cascading Bandit Structure}: Users are presented with an ordered list of items and view them sequentially. The interaction ends either upon the first click or when the session times out, and feedback is only observed up to that point to capture the partial observability common in recommendation systems.
  \item \textbf{Simultaneous Multi-user Interactions}: The platform serves many users concurrently, enabling a \emph{selective learning} scheme in which the platform adopts different exploration–exploitation strategies across users.
  \item \textbf{Ad-Level Heterogeneous Rewards}: Although each customer’s response is binary (i.e., whether a click occurs), the expected value of a positive response depends on the specific ad, reflecting heterogeneous per-click revenue.
\end{enumerate}

Despite the importance of each component, prior works addressed them only in isolation. In particular, there is no existing theoretical framework that jointly handles \emph{heterogeneous per-click rewards} and \emph{multi-user interactions} in a contextual cascading setting. To address these limitations, we propose the \textbf{Multi-User Cascading Contextual Bandit (MCCB)} framework, which models the interaction between a large-scale recommendation platform and multiple users arriving simultaneously. In each episode, the platform is presented with a batch of users, each described by their own contextual features. For each user, the platform selects a personalized ordered list of ads to display, drawn from a global catalog. We illustrate this application in Figure~\ref{fig:protocol}.

\begin{figure}[htbp]
\centering
\includegraphics[width=1\linewidth]{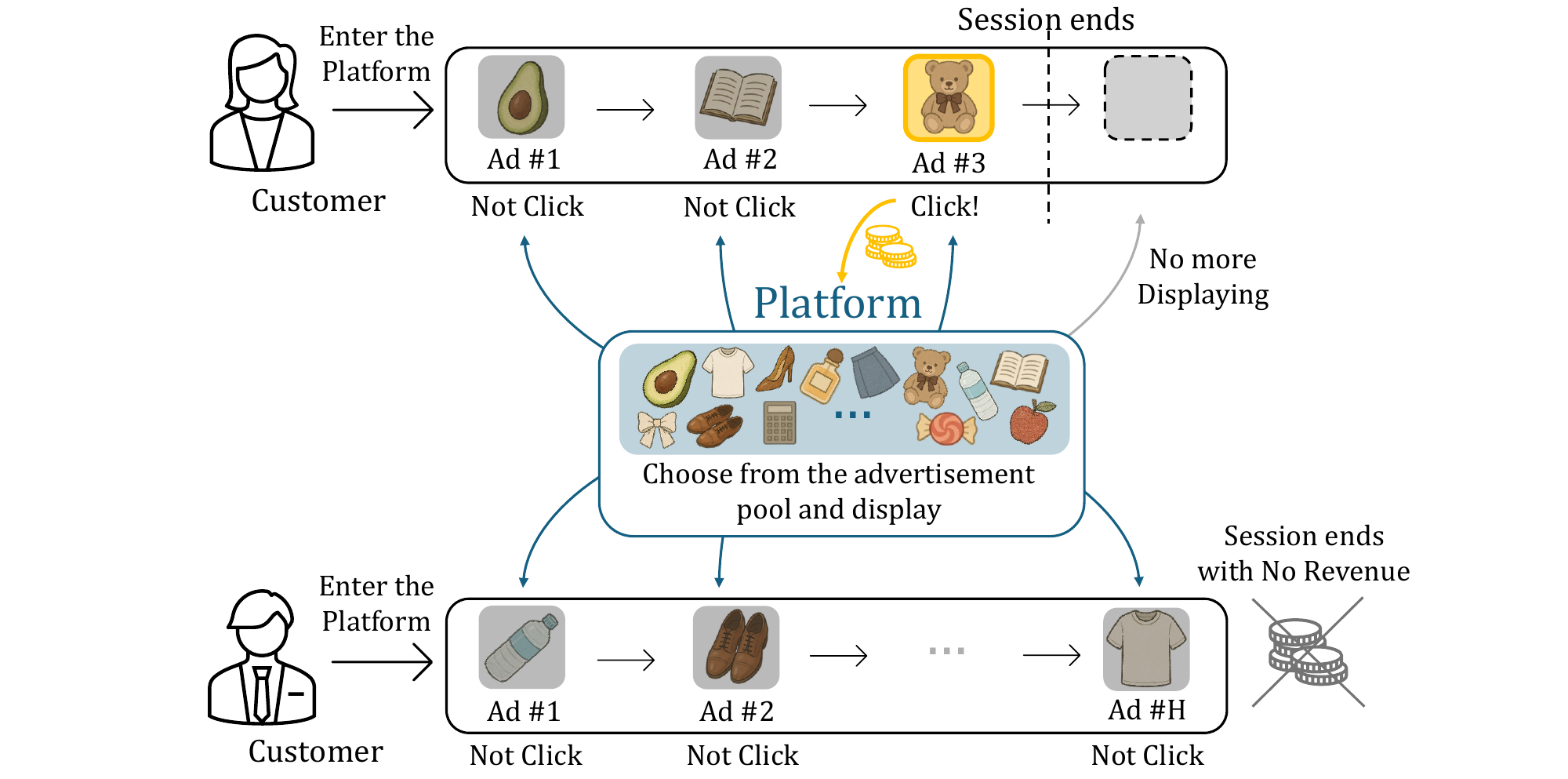}
\caption{
Illustration of one episode in the application scenario for ad recommendation. The platform operates over multiple episodes. In each episode, a batch of customers arrives and interacts with the system independently through their own sessions. When a customer enters the platform, they are presented with a personalized, ordered list of ads to view sequentially during the session. The \textbf{customer} may click on an ad or skip it. If a click occurs, the interaction ends immediately, and the platform receives the corresponding revenue. If no ad is clicked by the end of the list, the session terminates without generating revenue. Throughout the session, the \textbf{platform} selects and displays ads based on the customer’s characteristics. This process repeats across customers in each episode, and the platform’s objective is to maximize cumulative revenue over time.
}
\label{fig:protocol}
\end{figure}

\subsection{Contributions}

Our main contributions are summarized as follows:

\begin{itemize}
    \item We formulate a novel combinatorial bandit framework, the \textbf{Multi-user Cascading Contextual Bandit (MCCB)}, which generalizes classical cascading bandits by incorporating both parallel user interactions and heterogeneous, arm-specific rewards. This setting captures many practical scenarios in large-scale recommendation and ad-serving systems.
    \item We propose an Upper Confidence Bound (UCB) algorithm, termed \textbf{UCBBP}, and establish a $T$-episode regret upper bound of \( \widetilde{\mathcal{O}}(\sqrt{(d+K)HNT}) \) in Theorem~\ref{thm:UCBBP}, where $H$ refers to session length and $N$ refers to the number of contexts per episode.
    \item We propose a novel algorithm \textbf{AUCBBP} that performs \emph{selective exploration} by assigning UCB-based updates to only a subset of users, while exploiting greedily on the rest. This strategy significantly reduces the exploration overhead when the number of users $N$ is large. We achieve a regret bound of 
    $\widetilde{\mathcal{O}}\big(\sqrt{(d+K)(T+HN)}\ \big)$ in Theorem~\ref{thm:AUCBBP}, improving over the naive $\widetilde{\mathcal{O}}\big(\sqrt{(d+K)HNT}\big)$ dependence by disentangling the contributions of horizon and user scale. 
    \item We validate our algorithms through numerical experiments on synthetic datasets. The results show that our methods achieve time-averaged regret that converges to zero as episodes progress, in contrast to baseline algorithms whose regret plateaus. These findings support our theoretical analysis and demonstrate the effectiveness of our algorithms.
\end{itemize}

\section{Related Work}

\textbf{Personalized advertising}~\cite{li2010a, li2010b, LiUnbiasedrecommendation, 10.1145/1719970.1719976} plays a central role in online platforms, where systems must learn to recommend relevant ads tailored to each user's preferences. This challenge is further complicated by the need to make sequential decisions under uncertainty.
To address this, a broad range of bandit models has been studied~\cite{10.5555/3524938.3525762, gu2025evolutioninformationinteractivedecision, ba2024doublyoptimalnoregretonline}. Many of these models provide theoretical guarantees on regret. Notable examples include LinUCB~\cite{li2010a}, Thompson Sampling~\cite{pmlr-v28-agrawal13, pmlr-v23-agrawal12}, and generalized linear model (GLM)-based approaches~\cite{Filippi2010, 10.5555/3305890.3305895}, all of which have enabled robust learning across diverse application settings.

\textbf{Contextual bandits} frameworks are particularly attractive for personalized recommendation, as they incorporate user-specific information into sequential decision-making~\cite{pmlr-v15-chu11a, li2010a}. More recent work has focused on modeling more realistic user behavior by accounting for the sequential nature of user-item interactions and the structure of observed feedback. In the cascading bandit model, for instance, users are assumed to scan a ranked list of items and click on the first attractive one, leading to partial feedback. This model has been extensively studied~\cite{wen2025jointvalueestimationbidding, 10.1145/3627673.3679763, pmlr-v37-kveton15, Zongrecommendationproblems}, and extended to contextual and combinatorial variants~\cite{10.5555/3045390.3045522}, such as cascading assortment bandits~\cite{10.5555/3666122.3667137} and cascading hybrid bandits~\cite{10.1145/3383313.3412245}.

To the best of our knowledge, the closest work to ours is~\citet{du2024cascading}, which also considers a cascading bandit setting formulated as an MDP and proposes \textsc{CascadingVI}. Unlike absorbing bandit models, their setting allows the process to continue \emph{after a click}, with the environment transitioning to the next state. The algorithm predicts future states via estimated transition probabilities and recursively computes value functions.
In contrast, we adopt a contextual bandit framework, where the state at each step is defined by user-specific context and the decision step index. Moreover, our reward structure is designed to be absorbing, meaning that the episode terminates upon the first click, with no further transitions or delayed rewards.

In addition, most existing cascading bandit frameworks do not consider heterogeneous rewards within the cascading structure. Moreover, despite the practical significance of this concurrent, multi-user setting, it has received limited attention in the literature. These limitations reduce the applicability of existing models to modern ad platforms, which must optimize ad allocation at scale under both complex reward structures and simultaneous user interactions.

\section{Problem Setting}

We now formalize the Multi-User Cascading Contextual Bandit (MCCB) framework.
While the introduction used application-level terminology such as platform, customer, and ad, we now shift to a formal notation based on the standard contextual bandit setting. Specifically, we refer to the decision-maker as the \emph{learner}, each customer as a \emph{context}, and each ad as an \emph{arm} associated with a \emph{reward} capturing the ad revenue.

\subsection{Problem Formulation}

In MCCB, the learner operates over $T$ episodes, and in each episode, the learner interacts with \( N \) contexts and selects from a fixed set of $K$ arms.
The interaction with each context unfolds over an $H$-step session, during which the learner presents a personalized, ordered slate of up to H candidates (with replacement) from the global arm set $[K]$. This setting is typically seen in cascading bandits, where the candidates in the slate are displayed to the context sequentially. We consider a binary feedback signal from each context-arm pair along the session steps. If a positive signal is observed at position \( h \le H \), the session of the context terminates and the learner receives the associated reward of the arm. 
We refer to this event as the context being \emph{absorbed} at step $h$.
If no positive signal occurs after \( H \) steps, the session of the context ends with zero reward. We illustrate the MCCB framework in Figure~\ref{fig:MCCB}. When \( H = 1 \), MCCB reduces to a contextual bandit operating over multiple users in parallel. In an extreme case where \( N = 1 \) and \( H = 1 \), MCCB recovers the standard contextual bandit setting.

\begin{figure}[htbp]
\centering
\includegraphics[width=1\linewidth]{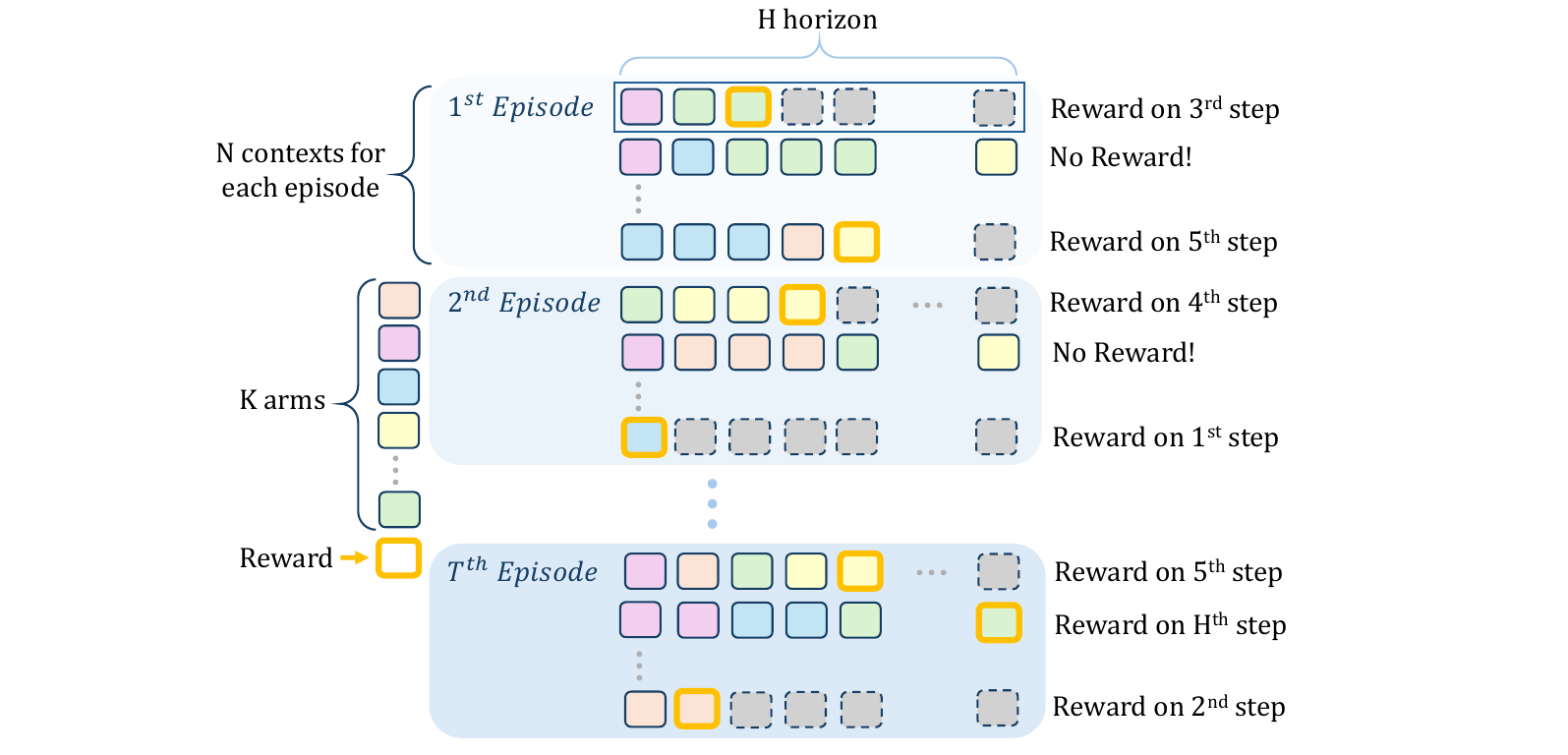}
\caption{
An overview of the Multi-User Cascading Contextual Bandit (MCCB) framework. In each episode, the learner interacts with $N$ contexts in parallel. Each context is sequentially presented with up to $H$ arms from a fixed arm set $[K]$ (chosen with replacement). If a context observes a positive signal at position $h \le H$, the session terminates and the learner immediately receives the corresponding reward associated with that arm. If none of the $H$ displayed arms observe a positive signal, the session ends without reward.
}
\label{fig:MCCB}
\end{figure}

For each arm $k \in [K]$, we assume a known reward value $e_k > 0$, which the learner receives upon a positive signal. Let $\mathcal{X} \subseteq \mathbb{R}^d$ denote the context space. At each episode $t \in [T]$, the platform observes $N$ contexts $x_{t,n} \in \mathcal{X}$, for $n \in [N]$.

We adopt a shared-parameter generalized linear model for click probabilities. For any context–arm pair \( (x_{t,n}, k) \), define the joint feature vector:
\[
z_{t,n,k} := [x_{t,n};\, \mathbf{k}] \in \mathbb{R}^{d+K},
\]
where \( \mathbf{k} \) is the one-hot encoding of arm \( k \). Let \( \theta \in \mathbb{R}^{d+K} \) denote the shared parameter vector. 

\subsubsection{Reward Function Assumption}

We assume that the probability of an absorption is modeled using a logistic function:
\[
\mathbb{P}(r_{t,n,k} = e_k \mid x_{t,n}, k) = \sigma(z_{t,n,k}^\top \theta),
\]
where \( \sigma(u) := \frac{1}{1 + e^{-u}} \) is the sigmoid function. That is, the reward is binary, and the probability of receiving a reward depends on the context–arm features via a logistic model. This modeling assumption has been commonly adopted in prior work on bandits with binary rewards~\cite{pmlr-v139-jun21a, 10.5555/3305890.3305895, 10.5555/3524938.3525224}.

\begin{remark}
The setting in which an arm can be displayed multiple times in a session is also practical, as each arm may represent a category of similar items with comparable characteristics in terms of reward and click probability. As a result, the same arm may appear more than once within a single session.
\end{remark}

\subsection{MDP Representation}

Since our problem involves sequential decision-making and arm-dependent rewards, we model the reward structure using a Markov Decision Process (MDP). Each episode corresponds to a finite-horizon MDP with deterministic transitions, where the learner's actions determine both immediate rewards and future continuation through the cascading structure. This formulation allows us to explicitly capture the impact of each arm selection on downstream value and to perform backward planning for optimal policy evaluation.

\subsubsection{State Space}
We define the \emph{state} at step $h$ of episode $t$ as the collection of individual user states:
\[
S_{t,h} = \left\{ S_{t,h}^{x_{t,n}} : n = 1, \dots, N \right\}.
\]
Each user-specific state \( S_{t,h}^{x_{t,n}} \) is given by:
\[
S_{t,h}^{x_{t,n}} =
\begin{cases}
  h & \text{if } x_{t,n} \text{ is still active at step } h, \\
  \texttt{absorbed} & \parbox[t]{.55\linewidth}{
    if $x_{t,n}$ has received a non-zero reward or $h > H$.
  }
\end{cases}
\]

\subsubsection{Action Space}
The action space at step $h$ is:
\[
A_{t,h} = \left\{k_{t,n,h} : n = 1, \dots, N \right\},
\]
where
\[
k_{t,n,h} =
\begin{cases}
  \{1, \dots, K, \textit{skip}\} & \text{if } x_{t,n} = h, \\
  \emptyset & \text{if } x_{t,n} = \texttt{absorbed}.
\end{cases}
\]

\subsubsection{History (Filtration)}
\label{subsubsec:history}
The filtration \( \mathcal{F}_t \) up to episode \( t \) consists of all observed tuples from episodes \( 1 \) to \( t \), where only pre-absorbing observations are included:
\[
\mathcal{F}_t = \left\{ (z_{s,n,k_{s,n,h}}, k_{s,n,h}, r_{s,n,h})   \;\middle|\; s \le t,\; n \in [N],\; h < h^*_{s,n} \right\}.
\]

It is updated episode by episode by appending new pre-absorbing observations.
Here, \( h^*_{s,n} \in [H] \) denotes the absorbing horizon for user \( n \) in episode \( s \). If no click occurs, we define \( h^*_{s,n} := H \).

\subsubsection{State-Value Function under Policy $\pi$}
For any deterministic policy $\pi$, the value of a state $S_{t,h}$ is defined recursively as:
\[
V^{\pi}(S_{t,h}) = Q^{\pi}( S_{t,h},\ \pi(S_{t,h}) ),
\]
where $Q^{\pi}(S, A)$ denotes the expected return starting from state $S$, taking action $A$, and thereafter following policy $\pi$.

\subsubsection{Optimal Policy}

The optimal policy $\pi^*$ maximizes the expected return at each decision step:
\[
\pi^{*}(S_{t,h}) = \arg\max_{A_{t,h}} Q( S_{t,h}, A_{t,h} ).
\]

\subsubsection{Expected Total Reward Under General Policy $\pi$}

The learner's objective is to maximize the total expected reward accumulated over all episodes under a given policy $\pi$:
\[
\mathbb{E}[R] = \sum_{t=1}^{T} V^{\pi}( S_{t,1} ),
\]
where $S_{t,1}$ denotes the initial state of episode $t$.

\subsubsection{Performance Metric}
We adopt the notion of \textit{pseudo-regret} to evaluate the learner’s performance when the underlying parameter $\theta$ is unknown.
\begin{definition}
The pseudo-regret of $T$ episodes is defined as
$$
\text{Regret}(T)
\;:=\;
\sum_{t=1}^{T}
\left(
    V^{\pi^*}( S_{t,1} )
    \;-\;
    V^{\pi_t}( S_{t,1} )
\right).
$$ 
\end{definition}
We aim to minimize this regret in a setting where rewards are arm-specific and episodes terminate upon receiving a positive reward, inducing an absorbing structure in the user state dynamics.

\section{Algorithm: UCBBP}

In this work, we first focus on a setting where arms are displayed sequentially over a finite horizon $H$, and each arm is associated with a heterogeneous reward. It means that the probability of being chosen and the reward amount can vary across arms. Even when the expected rewards are similar, an item with a high success probability but low reward and another with low probability but high reward pose different implications for display strategy.

We interpret the horizon $H$ as the number of display opportunities available to the learner. To maximize total revenue, a natural strategy is to prioritize arms with higher potential reward but lower success probability earlier in the sequence, and defer more reliable but lower-reward arms to later positions.
We propose the Upper Confidence Bound with Backward Planning(Algorithm~\ref{alg:ucbbp}), which optimizes arm selection in cascading environments by combining statistical confidence and dynamic planning. This algorithms is based on two key components described in the following subsections: the construction of upper confidence bounds and a finite-horizon backward planning procedure.

\begin{algorithm}[tb]
\caption{UCB with Backward Planning}
\label{alg:ucbbp}
\KwIn{Episodes $T$, horizon $H$, number of arms $K$, regularization $\lambda$, warm-up episodes $T_0$, confidence radius $\beta_t$}
\textbf{Initialize:} $\hat{\theta} \gets 0$, $A \gets \lambda I$, $b \gets 0$\;
\For{$t \gets 1$ \KwTo $T$}{
  observe user contexts $\mathcal{X}^{(t)} \gets \{x_{t,1}, \dots, x_{t,N}\}$\;
  $\texttt{absorbed}[n] \gets \texttt{F}$ for all $n$\;
  $\mathcal{Z} \gets \emptyset$, $\mathcal{R} \gets \emptyset$\;

  \For{$h \gets 1$ \KwTo $H$}{
    \For{$n \gets 1$ \KwTo $N$ \KwSty{s.t.} $\texttt{absorbed}[n] = \texttt{F}$}{
      compute $\widehat{Q}(x_{t,n,h}, k)$ for all $k$ using $\hat{\theta}$\;
      \uIf{$t \le T_0$}{
        select $k_{t,n,h}$ via round-robin\;
      }
      \Else{
        $k_{t,n,h} \gets \arg\max_k \ \mathcal{U}_{t,n,h}(k)$ \tcp*{see \eqref{eq:ucb_def}}
      }
      observe reward $r_{t,n,h} \in \{0, e_k\}$\;
      append $z_{t,n,k_{t,n,h}}$ and $r_{t,n,h}$ to $\mathcal{Z}$ and $\mathcal{R}$\;
      \If{$r_{t,n,h} > 0$}{
        $\texttt{absorbed}[n] \gets \texttt{T}$\;
      }
    }
  }

  \ForEach{$(z, r) \in (\mathcal{Z}, \mathcal{R})$}{
    identify arm $k$, normalize reward $y \gets r / e_k$\;
    $\hat{p} \gets \sigma(z^\top \hat{\theta})$, $w \gets \hat{p}(1 - \hat{p})$\;
    $A \gets A + w z z^\top$\;
    $b \gets b + z (y - \hat{p})$;
  }
  $\hat{\theta} \gets A^{-1} b$;
}
\end{algorithm}

\subsection{MDP based Backward Planning}\label{subsec:backward_planning}

To account for arm-dependent rewards and the cascading structure of opportunities, UCBBP employs a backward-planning procedure over a finite-horizon MDP of length $H$. For each context \( x \), display position \( h \in [H] \), and arm \( k \in [K] \), the algorithm computes a value estimate \( \widehat{Q}(x, k) \) that incorporates both immediate and future expected rewards.

\subsubsection{Parameter Estimation of $\hat{\theta}$}
At the end of each episode, the parameter estimate is updated from the collected data. 
Each reward is normalized as $y = r/e_k$, and the predicted success probability is 
$\hat{p} = \sigma(z^\top \hat{\theta})$ with weight $w = \hat{p}(1-\hat{p})$. 
The updates are
\[
A \leftarrow A + wzz^\top, \qquad b \leftarrow b + z(y-\hat{p}),
\]
and the new estimate is
\[
\hat{\theta} \leftarrow A^{-1} b.
\]
This constitutes a single Iteratively Reweighted Least Squares(IRLS) step, serving as an online approximation to the logistic MLE.

\subsubsection{Algorithm for Calculating $\hat Q_h(x,k)$ with $\hat\theta$}\label{append:back_planning}
The procedure for computing $\hat Q_h(x,k)$ given the current parameter estimate~$\hat{\theta}$ 
follows a backward value iteration over the remaining horizon, 
starting from the terminal step and propagating values back to the current step~$h$.
The full pseudocode is presented in Algorithm~\ref{alg:backward_value_iteration_corrected}.
\begin{algorithm}[tb]
\caption{Backward Value Iteration (for Current Step $h$)}
\label{alg:backward_value_iteration_corrected}
\KwIn{Context $x$; parameter estimate $\hat{\theta}$; rewards $\{e_k\}_{k=1}^K$; current step $h$ ($1\le h\le H$); total horizon $H$}
\KwOut{$\bigl\{\hat Q_{h}(x,k)\bigr\}_{k=1}^K$}

\textbf{Initialize:} $\hat V_{H+1}(x) \gets 0$\;

\For{$t \gets H$ \textbf{down to} $h$}{
  \For{$k \gets 1$ \KwTo $K$}{
    $f_k(x) \gets \sigma\!\bigl(z_k^\top \hat{\theta}\bigr)$\;
    $\hat Q_{t}(x,k) \gets f_k(x)\,e_k \;+\; \bigl(1-f_k(x)\bigr)\,\hat V_{t+1}(x)$\;
  }
  $k^* \gets \arg\max_{k}\;\hat Q_{t}(x,k)$\;
  $\hat V_{t}(x) \gets \hat Q_{t}(x, k^*)$\;
}

\textbf{return} $\bigl\{\hat Q_{h}(x,k)\bigr\}_{k=1}^K$\;
\end{algorithm}

The MDP interprets display positions as steps, with user feedback determining transitions. If a user clicks on an item, the episode terminates. Otherwise, the context proceeds to the next display position. The immediate expected reward is computed as:
\[
\widehat{r}(x_{t,n}, k) = \widehat{p}(x_{t, n}, k) \cdot e_k,
\]
where \( \widehat{p}(x_{t, n}, k) \) denotes the estimated click probability and \( e_k \) is the reward received upon a click.
The backward recursion proceeds from position \( h = H \) down to \( h = 1 \), with the terminal value \( \widehat{V}_{H+1}(x_{t,n}) = 0 \) defined as:
\[
\widehat{Q}(x_{t,n,h}, k) = \widehat{r}(x_{t,n}, k) + \left(1 - \widehat{p}(x_{t,n}, k)\right) \cdot \widehat{V}_{h+1}(x_{t,n}).
\]

Note that the absorbing probability \( \widehat{p}(x, k) \) is estimated by \( \hat{\theta} \).
Through this backward dynamic programming process, the estimated value \( \widehat{Q}(x_{t, n, h}, k) \) captures the long-term utility of displaying arm \( k \) at position \( h \), considering both immediate reward value and future opportunities. These \( \widehat{Q} \) values are then used in the UCB framework described in the previous section to guide exploration and exploitation.

\subsubsection{Special Case: Opposite Ordering of Reward and Probability.}
When the arms are ordered so that click probabilities are strictly decreasing 
while rewards are strictly increasing, the $Q$–function exhibits additional structure 
that aligns with the backward–planning intuition.
The following proposition shows that, for any fixed context and horizon, 
the $Q$–values across arms form a unimodal sequence. 
Moreover, the next proposition establishes that the optimal arm index is non–increasing 
as the remaining horizon shortens, i.e., high–reward/low–probability arms are placed 
earlier, while high–probability/low–reward arms are deferred to later positions.

\begin{proposition}[Unimodality of the $Q$–function over arms]
\label{prop:unimodal}
Fix a context $x\in\mathcal X$ and a horizon
$h\in\{1,\dots,H\}$.
Assume the arms are indexed so that
\[
f_{1}(x)\;\ge\;f_{2}(x)\;\ge\;\dots\;\ge\;f_{K}(x),\qquad
e_{1}\;\le\;e_{2}\;\le\;\dots\;\le\;e_{K},
\]
where $f_k(x) = \sigma(z_k^\top \theta) \in (0,1)$ with 
$z_k$ denoting the concatenation of context $x$ and the one-hot encoding of arm $k$.

Define
\[
Q_h(x, k)=f_k(x)\,e_k + (1-f_k(x))\,V^\star_{h+1}(x),
\qquad k=1,\dots,K,
\]
with continuation value $V^*_{h+1}(x)\ge 0$ independent of $k$.
Then there exists $k^\star \in \{1,\dots,K\}$ such that
\[
Q_h(x, 1)\;\le\;Q_h(x, 2)\;\le\;\dots\;\le\;Q_h(x, k^*)
\;\ge\;
Q_h(x, k^*\!+\!1)\;\ge\;\dots\ge\;Q_h(x, K),
\]
i.e.\ the sequence $\{Q_h(x, k)\}_{k=1}^K$ is unimodal.
\end{proposition}

\begin{proposition}[Monotonicity in Horizon]
\label{prop:monotone_arm_choice}
Under the same ordering of probabilities and rewards as above,
let
\[
k_h^*(x) = \arg\max_{k \in [K]} Q_h(x, k), \quad
Q_h(x, h) = f_k(x)\,e_k + (1 - f_k(x))\,V^*_{h+1}(x),
\]
with $V^*_{H+1}(x) = 0$.
Then the optimal arm indices are non-increasing as the horizon shrinks:
\[
k_H^\star(x) \ge k_{H-1}^\star(x) \ge \dots \ge k_1^\star(x),
\qquad \text{for all } x \in \mathcal X.
\]    
\end{proposition}

\noindent
Proofs of these results are deferred to Appendix~\ref{app:prop_proofs}.

\begin{figure}[htbp]
    \centering
    \includegraphics[width=0.8\linewidth]{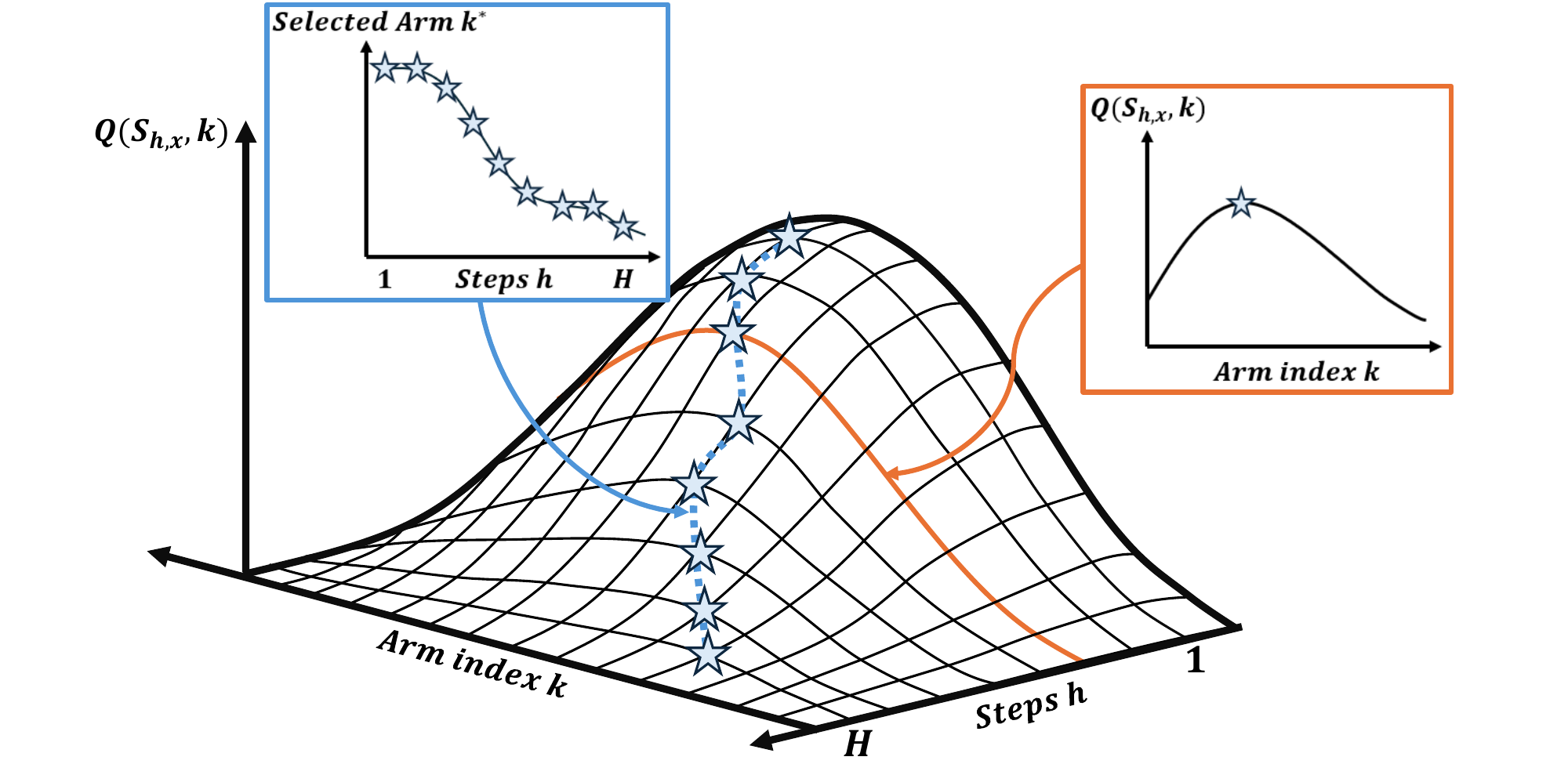}
    \caption{
        Illustration of the structural properties in the opposite-ordering case.
        The surface plot depicts $Q(S_{h,x}, k)$ values over arm index $k$ and steps $h$.
        \textbf{Right inset:} Unimodality of $Q(S_{h,x}, k)$ in $k$ (Proposition \ref{prop:unimodal}).
        \textbf{Left inset:} Monotonic decrease of the optimal arm index $k^*$ as $h$ decreases (Proposition \ref{prop:monotone_arm_choice}).
    }
    \label{fig:backward-structure}
\end{figure}

\subsection{Upper Confidence Bounds and Confidence Set}
In the UCBBP, for each episode $t \in [T]$, context $n \in [N]$, and horizon $h \in [H]$, 
the platform computes an optimistic value estimate for each arm $k \in [K]$, 
and selects the arm with the highest value. The selected arm $k_{t,n,h}$ is given by
\[
k_{t,n,h}
= \arg\max_k \left[
  \widehat{Q}(x_{t,n,h}, k)
  + \beta_{t-1}\cdot\sqrt{z_{t,n,k}^\top A^{-1} z_{t,n,k}}
\right],
\]    
where $\widehat{Q}(x_{t,n,h}, k)$ denotes the estimated value of arm $k$ for context $x_{t,n,h}$ 
based on the backward planning procedure in Section~\ref{subsec:backward_planning}. 
The vector $z_{t,n,k}$ represents the joint feature vector associated with the context–arm pair $(x_{t,n,h}, k)$. 
The matrix $A$ is the regularized covariance (information) matrix constructed from the history filtration 
$\mathcal{F}_{t-1}$ defined in Section~\ref{subsubsec:history}. 
The confidence radius $\beta_{t-1}$ is set to ensure a high-probability bound on the estimation error.

Although the algorithm does not explicitly maintain the confidence set
\[
\mathcal{C}_t(\delta) := \left\{ \theta \in \mathbb{R}^d : \|\hat\theta_t - \theta\|_A \le \beta_{t-1} \right\},
\]
the choice of \( \beta_t \) guarantees that, with high probability, the adjusted value
\begin{equation}
\mathcal{U}_{t,n,h}(k) := \widehat{Q}(x_{t,n,h}, k)
+ \beta_{t-1} \cdot \sqrt{z_{t,n,k}^\top A^{-1} z_{t,n,k}}
\label{eq:ucb_def}
\end{equation}
serves as an upper bound on the true value \( Q^*(x_{t,n,h}, k^*) \). We later quantify the validity of this upper bound in Lemmas \ref{lem:action-q-bound} and \ref{lem:HfreeQbound}. This enables UCBBP to effectively balance exploration and exploitation through the principle of optimism in the face of uncertainty.

\section{Regret Analysis of UCBBP}
\subsection{Assumptions}
\begin{description}
  \item[(A1)] \textbf{Bounded feature vectors.}  
  For all $(t,n,h)$ and corresponding joint feature vector $z_{t,n,k}$, we have:
  \[
    \| z_{t,n,k} \|_2 \le c_z.
  \]
  \item[(A2)] \textbf{Isotropic Gaussian contexts.}  
  Contexts are i.i.d.\ draws from a standard Gaussian:
  \[
    x_{t,n} \sim \mathcal{N}(0, I_d),
    \qquad \text{independently over } (t,n).
  \]
  \item[(A3)] \textbf{No Universally Unattractive Item — Positive Margin.}
For every planning step $h\in[H]$ and every arm $k\in[K]$ there exist constants $\Delta_0>0$ and $\gamma_0>0$ such that
\[
  \Pr_{x\sim\mathcal D}
      \Bigl[
        k=\arg\max_{j\in[K]} Q_h\bigl(z_{t, n, j};\theta\bigr)
        \;\text{and}\;
        Q_h\bigl(z_{t, n, k};\theta\bigr)
        -\max_{j\ne k}Q_h\bigl(z_{t, n, j};\theta\bigr)
        \ge \Delta_0
      \Bigr]
  \;\ge\;\gamma_0.
\]
Thus every item is the uniquely $\Delta_0$‑margin optimal recommendation for a non‑zero measure of user contexts.
\end{description}

\begin{remark}\label{assumption}
\textbf{Assumption~(A1)} ensures that the regret bound is invariant to feature scaling and is standard in linear contextual bandits~\cite{AbbasiYadkori2011Linear, li2010a}.
\textbf{Assumption~(A2)}, which posits i.i.d.\ standard Gaussian contexts, is a common design choice in theoretical analyses of contextual bandits~\cite{AbbasiYadkori2011Linear, Filippi2010}. It ensures sub-Gaussian concentration for the design matrix and facilitates the construction of confidence sets.
\textbf{Assumption~(A3)} imposes a positive margin condition, ensuring that each arm is the uniquely optimal choice for some user contexts. This is analogous to identifiability or margin assumptions occasionally used in cascading contextual bandits~\cite{10.5555/3666122.3667137, cheung2017assortmentoptimizationunknownmultinomial}, and reinforcement learning~\cite{dann2021beyond}, which support efficient learning under uncertainty. Our formulation most closely resembles margin-based conditions in linear prediction~\cite{10.5555/2981780.2981879}, where similar assumptions facilitate the growth of the information matrix and enable tight regret bounds.

\end{remark}

\subsection{Regret Bound of UCBBP}
\begin{theorem}[Regret Bound of UCBBP]
\label{thm:UCBBP}
Suppose Assumptions \textnormal{(A1)}–\textnormal{(A5)} hold. Then for all \(T \ge 1\),
\begin{flalign*}
\mathbb{P}\Bigg(
\text{Regret}(NT)
&\;\le\;
N\,T_{0}\,e_{\max}
\;+\;
\frac{4\,e_{\max}\,k_{\mu}}{c_{\mu}\,\hat\sigma_{\min}}\;
\beta_{\widetilde T}\,
\sqrt{N\,T}\times
\Big[
\sqrt{\mathcal{V}_{\widetilde{T}}}
+
\beta_{\widetilde T}\,
\sqrt{\widetilde T\,\mathcal{V}_{\widetilde{T}}}
\Big]
\Bigg)
\;\ge\;
1 - \delta.
\end{flalign*}
where
\[
\beta_{\widetilde T}
=
\kappa\,e_{\max} c_z\,
\sqrt{
  2(d+K)\log \widetilde T \cdot \log\!\left(\frac{d+K}{\delta}\right)
}, \quad
\kappa
=
\sqrt{3 + 2\log\!\left(1 + \frac{c_z^2}{8\lambda}\right)}, 
\quad
c_\mu
=
\min_{|u| \le B} \sigma'(u) > 0
\]
\[
w_{\min}
=
\min_{t,n,h} w_{t,n,h}, \quad
p_{\min} =
\min_{t,n,h,k} \mathbb{P}\big(k \text{ is selected at } (t,n,h)\big),
\]
\[
\mathcal{V}_{\widetilde{T}}
=
\frac{4\,c_{z}^{2}}{w_{\min}p_{\min}}\,
\log\!\left(
  1+\frac{w_{\min}p_{\min}}{2\lambda}\,\widetilde T
\right), \quad
\widetilde{T}
=
\sum_{t=1}^{T} \sum_{n=1}^{N} h_{t,n}^{*}, 
\quad\text{and}\quad \widetilde{T} \le NTH.
\]

\end{theorem}

\subsection{Proof Outline}

Assumptions~(A4) and (A5) are much milder and thus we list and discuss them in Appendix~A.

We begin by recalling the definition of cumulative regret over \( T \) episodes:
\[\text{Regret}(T)
\;:=\;
\sum_{t=1}^{T}
\left(
    V^{\pi^*}( S_{t,1} )
    \;-\;
    V^{\pi^{ucb}}( S_{t,1} )
\right).\]
In our setting with $N$ users per episode and context $x_{t,n}$ for user $n$, the regret can be decomposed into the sum of per-context regret:
\[\text{Regret}(T)
=
\sum_{t=1}^{T} \sum_{n=1}^{N}
\bigl|
V_1^{\pi^*}(x_{t,n})
-
V_1^{\pi^{ucb}}(x_{t,n})
\bigr|,\]
where $V_1^{\pi^{ucb}}(x)$ denotes the expected reward starting from horizon $h=1$ under the UCBBP policy and context $x$.

Our key technique is to bound the per-context value function difference by decomposing it into three terms:
\[
  \bigl|
    V_h^{\pi^*}(x)
    -V_h^{\pi^{ucb}}(x)
  \bigr|
  \;=\;
  \bigl|
    Q_h^*(x, k^{*})
    -Q_h^{*}(x, k^{\mathrm{ucb}})
  \bigr|
\]
\[
  \le\;
  \underbrace{\bigl|
      Q_h^{*}(x, k^{*})
      -\hat Q_h(x, k^{*})
    \bigr|}_{\text{(value deviation of }k^{*}\text{)}}\;+\;
  \underbrace{\bigl|
      \hat Q_h(x, k^{*})
      -\hat Q_h(x, k^{\mathrm{ucb}})
    \bigr|}_{\text{(UCB gap)}}\;+\;
  \underbrace{\bigl|
      \hat Q_h(x, k^{\mathrm{ucb}})
      -Q_h^{*}(x, k^{\mathrm{ucb}})
    \bigr|}_{\text{(value deviation of }k^{\mathrm{ucb}}\text{)}}.
\]

We bound the two value deviation terms using the recursive structure of the $Q$-function in Lemma \ref{lem:action-q-bound}.

\begin{lemma}[Per-Action Q-Error Propagation]
\label{lem:action-q-bound}
For any episode $t \ge T_o$, context–step pair $(n, h)$, and arm $k \in [K]$, the following bound holds:
\[
\bigl| Q_h^{*}(x_{t,n}, k) - \hat Q_h(x_{t,n}, k) \bigr| \]
\[ \;\le\;
\bigl|
  \sigma\!\bigl(z_{t,n,k}^{\!\top}\theta\bigr)
  - \sigma\!\bigl(z_{t,n,k}^{\!\top}\hat\theta_t\bigr)
\bigr|
\cdot
\bigl(e_k + V_{h+1}^{*}(x_{t,n})\bigr) \ +\ \bigl(1 - \sigma\!\bigl(z_{t,n,k}^{\!\top} \hat\theta_t\bigr)\bigr)
\cdot \bigl| Q_{h+1}^{*}(x_{t,n}, k) - \hat Q_{h+1}(x_{t,n}, k) \bigr|.
\]
\end{lemma}

To avoid recursive expansion, we also provide a horizon-free bound in Lemma \ref{lem:HfreeQbound}.

\begin{lemma}[Horizon-Free Q-Error Bound]
\label{lem:HfreeQbound} 
Fix any episode $t \ge T_o$, step $h \in [H]$, and context $x_{t,n}$. Under assumptions \textnormal{(A4)}–\textnormal{(A5)}, the Q-function estimation error satisfies
\[
\bigl|Q_h^{*}\!\bigl(x_{t,n}, k\bigr) - \hat Q_h\!\bigl(x_{t,n}, k\bigr) \bigr|
\;\le\;
\frac{2e_{\max}}{
      \min_{k \in [K]}
      \sigma\!\bigl(z_{t,n,k}^{\!\top} \hat\theta_t\bigr)
}
\cdot
\max_{k \in [K]}
\bigl| \sigma\!\bigl(z_{t,n,k}^{\!\top} \theta\bigr)
     - \sigma\!\bigl(z_{t,n,k}^{\!\top} \hat\theta_t\bigr) \bigr|.
\]
\end{lemma}

Together, Lemmas~\ref{lem:action-q-bound} and~\ref{lem:HfreeQbound} show that the Q-function value deviation terms can be controlled by a parameter error term of the form $\lvert \sigma - \hat\sigma \rvert$, as commonly seen in standard contextual bandit analysis.
The remaining UCB gap term can be bounded directly using the confidence radius definition in the algorithm. Specifically, the difference in estimated Q-values between any two arms is upper bounded by the sum of their confidence widths.

We provide the supporting lemmas and detailed proof in Appendices \ref{appendix:Lemmas} and \ref{AppedixproofTh1}.

\section{Enhanced Algorithm: AUCBBP}

We consider a high-information regime where, in each episode, the algorithm interacts with $N$ users in parallel. Since episode $t$ yields at least $N$ observations, the parameter estimate $\hat\theta_t$ converges significantly faster than in classical single-context settings.

We propose our second algorithm, termed Active UCB with Backward Planning (Algorithm~\ref{alg:active-ucbbp}). After an initial warm-up phase, the algorithm switches to UCBBP but applies the UCB rule only to the $M_t$ contexts with the largest $z^\top A_t^{-1} z$ scores, since $z^\top A_t^{-1} z$ equals the predictive variance of $z^\top\hat\theta_t$ and larger values therefore yield greater information gain. All remaining contexts are handled greedily by selecting the arm with the highest current estimated reward.
We set
\[
  M_t \;=\; \max\!\Bigl\{1,\;\bigl\lfloor N\,e^{-t/\ln T}\bigr\rfloor\Bigr\},
\]
so, as illustrated in Figure~\ref{fig:protocol3}, the number of contexts governed by UCB decays exponentially with the episode index~$t$, boosting sample efficiency without sacrificing theoretical guarantees.

\begin{figure}[htbp]
\centering
\includegraphics[width=1\linewidth]{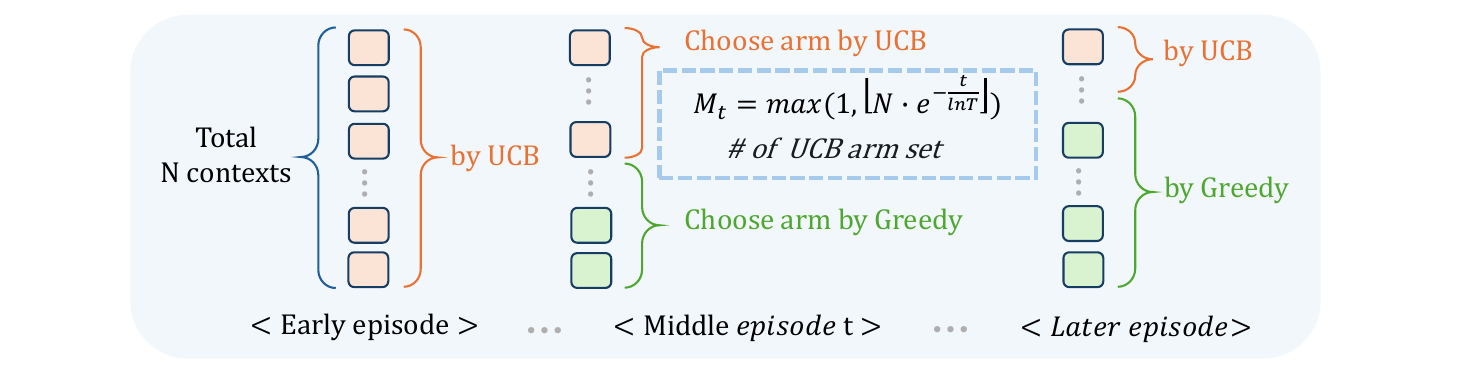}
\caption{
Overview of AUCBBP. It illustrates how, as learning progresses, most of the $N$ contexts in each episode are initially allocated to UCB-based exploration, ensuring sufficient exploration in the early phase. As time goes on, the number of contexts assigned to UCB decreases exponentially, thereby gradually shifting the focus toward exploitation.
}
\label{fig:protocol3}
\end{figure}

This structure enables efficient allocation of exploration efforts and minimizes redundancy for already well-learned arms. In particular, when the number of users $N$ is large, this selective strategy accelerates convergence and leads to improved cumulative rewards.

\begin{algorithm}[tb]
\caption{Active UCB with Backward Planning}
\label{alg:active-ucbbp}
\KwIn{Episodes $T$, horizon $H$, number of arms $K$, regularization $\lambda$, warm-up $T_0$, confidence radius $\beta_t$}
\textbf{Initialize:} $\hat{\theta} \gets 0$, $A \gets \lambda I$, $b \gets 0$\;

\For{$t \gets 1$ \KwTo $T$}{
  observe user contexts $\mathcal{X}^{(t)} \gets \{x_{t,1},\dots,x_{t,N}\}$;
  $\texttt{absorbed}[n]\gets \texttt{False}$ for all $n$;
  $\mathcal{Z}\gets \emptyset$, $\mathcal{R}\gets \emptyset$;

  \For{$h \gets 1$ \KwTo $H$}{
    \For{$n \gets 1$ \KwTo $N$}{
      compute $\widehat{Q}(x_{t,n,h},k)$ for all $k$ using $\hat{\theta}$\;
      $k_{\mathrm{ucb}}(x_{t,n},h)\gets \arg\max_k \mathcal{U}_{t,n,h}(k)$\;
      $k_{\mathrm{greedy}}(x_{t,n},h)\gets \arg\max_k \widehat{Q}(x_{t,n,h},k)$\;
      let $z \gets [x_{t,n},\,k_{\mathrm{ucb}}]$, compute $s_n \gets z^\top A^{-1} z$\;
    }
    $M_t \gets \max\!\left(1,\left\lfloor N\, e^{-t/\ln T}\right\rfloor\right)$\;
    let $\mathcal{T}_h$ be the indices of the top $M_t$ users with highest $s_n$\;

    \For{$n \gets 1$ \KwTo $N$}{
      \If{$\texttt{absorbed}[n]=\texttt{False}$}{
        \uIf{$t \le T_0$}{
          choose $k_{t,n,h}$ via round-robin\;
        }
        \uElseIf{$n \in \mathcal{T}_h$}{
          $k_{t,n,h} \gets k_{\mathrm{ucb}}(x_{t,n},h)$\;
        }
        \Else{
          $k_{t,n,h} \gets k_{\mathrm{greedy}}(x_{t,n},h)$\;
        }
        observe reward $r_{t,n,h}\in\{0,\,e_{k_{t,n,h}}\}$\;
        append $z_{t,n,k_{t,n,h}}$ and $r_{t,n,h}$ to $\mathcal{Z}$ and $\mathcal{R}$\;
        \If{$r_{t,n,h} > 0$}{
          $\texttt{absorbed}[n]\gets \texttt{True}$\;
        }
      }
    }
  }

  \For{$(z,r) \in (\mathcal{Z},\mathcal{R})$}{
    identify arm $k$, and normalize reward $y \gets r/e_k$\;
    $\hat p \gets \sigma(z^\top \hat{\theta})$, $w \gets \hat p(1-\hat p)$\;
    $A \gets A + w\, z z^\top$, \quad $b \gets b + z\, (y - \hat p)$\;
  }
  $\hat{\theta} \gets A^{-1} b$\;
}
\end{algorithm}

\section{Regret Analysis of AUCBBP}
\subsection{Assumptions}
We adopt the same assumptions as in the UCBBP setting.

\subsection{Regret Bound of AUCBBP}
\begin{theorem}[Total Regret of AUCBBP]
\label{thm:AUCBBP}
Suppose Assumptions \textnormal{(A1)}–\textnormal{(A5)} hold. Then for all \(T \ge 1\),
\begin{align*}
\mathbb{P}\Bigg(
&\operatorname{Regret}(NT)
\;\le\;\,
T_0 N e_{\max} 
+ 
\frac{4\,e_{\max}\,k_{\mu}}{c_{\mu}\,\hat\sigma_{\min}}\,
\beta_{\widetilde T}\,
\sqrt{\sum_{t=1}^{T} M_t \cdot \mathcal{V}_{\widetilde{T}}}  \\
&+ 2\,\beta_{\widetilde T}\,
\sqrt{H \sum_{t=1}^{T} M_t \cdot \mathcal{V}_{\widetilde{T}}} + 
\frac{4e_{\max} \,c_z}
     {c_\mu \hat\sigma_{\min} \sqrt{\tfrac12 w_{\min}p_{\min}}}
\cdot \sqrt{T} \cdot  \beta_{\widetilde T}
\Bigg)\\
&\;\ge\;
1 - \delta.
\end{align*}
where \( M_t := \max\left\{1,\left\lfloor N e^{-t/\ln T} \right\rfloor\right\} \), and all other quantities are defined in Theorem~\ref{thm:UCBBP}.

\end{theorem}

\subsection{Proof Outline}
The total regret can be decomposed into regret originated the UCB policy and regret originated from the greedy policy. As the algorithm progresses, the majority of contexts are eventually assigned to the greedy policy. Therefore, the key idea in the analysis is to derive a tighter bound on the regret incurred by the greedy policy. We provide the detailed proofs and supporting lemmas in Appendix~\ref{appendix:Lemmas} and~\ref{Appendixproofth2}.

\begin{lemma}[Greedy Regret Bound in Terms of Value Difference]
\label{lem:greedy-regret}
Let $k^*$ be the optimal arm under $\theta$ and $k^{\mathrm{grd}}$ be the greedy arm selected using $\hat\theta_t$. Under assumptions \textnormal{(A4)}–\textnormal{(A5)}
\[
  V_1^{\pi^*}(x_{t,n}) - V_1^{\pi^{grd}}(x_{t,n}) 
  \;\le\;
  \frac{e_{\max}}{\hat\sigma_{\min}}\left(
    \left| z_{t,n,k^*}^\top(\theta - \hat\theta_t) \right|
    +
    \left| z_{t,n,k^{\mathrm{grd}}}^\top(\hat\theta_t - \theta) \right|
  \right).
\]
\end{lemma}

Applying the same value function decomposition technique as in UCBBP, the UCB gap term vanishes under the greedy policy, allowing for a tighter bound. This is a key advantage in analyzing the regret of the greedy component.

\section{Numerical Experiments}

In this section, we evaluate the performance of our proposed algorithms, \textbf{UCBBP} and \textbf{AUCBBP}, through numerical experiments.
We report two types of regret in Figure \ref{fig:numerical}:
\begin{itemize}
    \item \textbf{Time-averaged regret:} Cumulative regret divided by the number of episodes $t$ (first row).
    \item \textbf{Context-averaged regret:} Cumulative regret divided by the number of contexts per episode $N$ (second row).
\end{itemize}

\begin{figure}[htbp]
\centering
\includegraphics[width=1\linewidth]{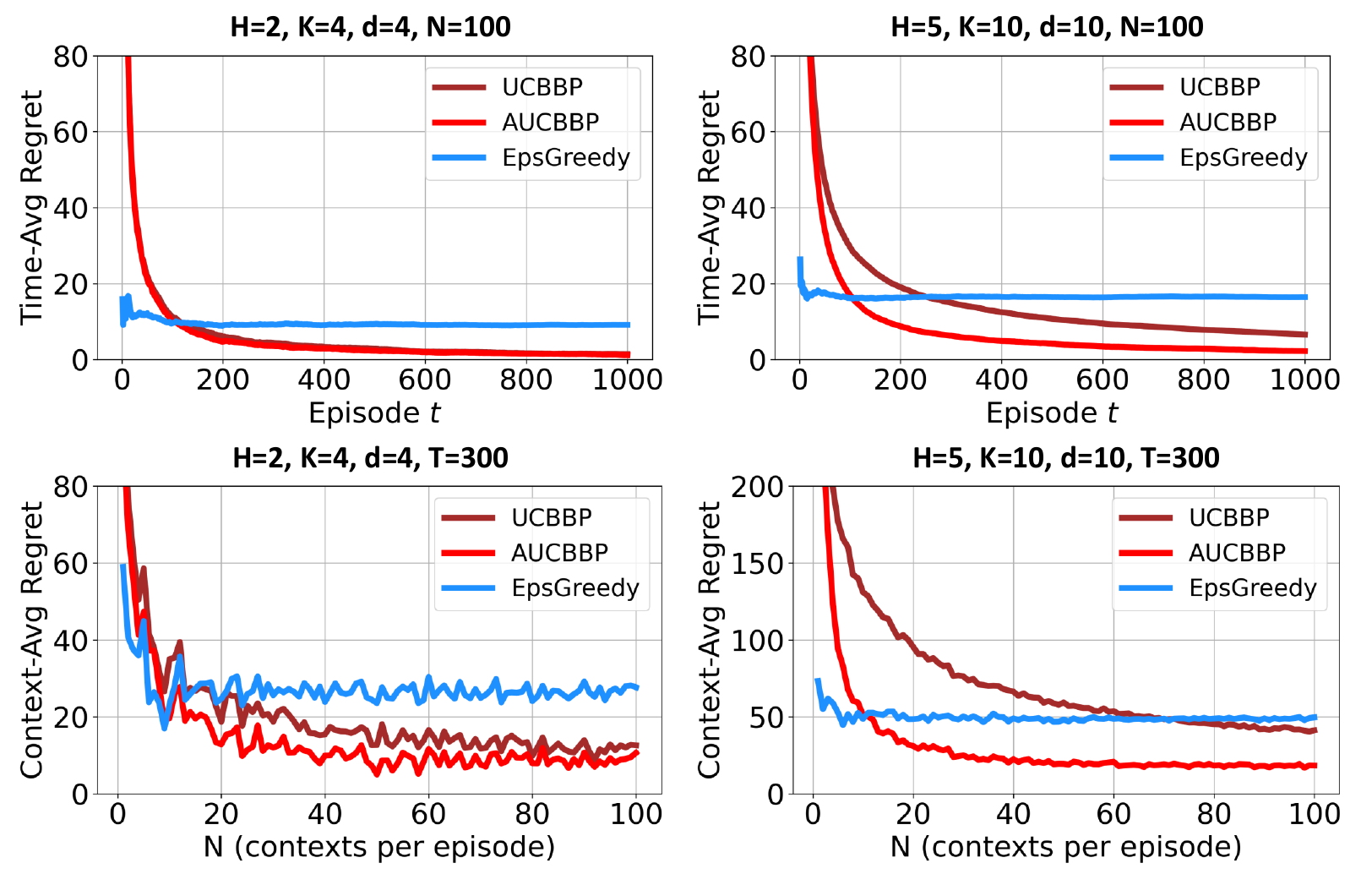}
\caption{
These plots show the results of our numerical experiments. In each episode $t$, $N$ context vectors $\{x_{t,n}\}_{n=1}^N \sim \mathcal{N}(0, I_d)$ are sampled, and rewards are generated using a logistic model. All results are averaged over 10 random seeds.}
\label{fig:numerical}
\end{figure}

In the first row of Figure \ref{fig:numerical}, the time-averaged regret of both UCBBP and AUCBBP shrinks to zero as learning proceeds, i.e., as $t$ increases. This empirically confirms that the cumulative regret grows sublinearly with $t$, which is consistent with our theoretical regret bounds.

In the second row, which reflects the numerical results from multiple instances, we observe that the context-averaged regret also shrinks to zero with increasing $N$. This suggests that our algorithms effectively leverage the higher number of user interactions in each episode. The decrease in context-averaged regret with $N$ further validates the robustness of our algorithms in high-information environments. 

While both algorithms benefit from the multi-user setting, AUCBBP exhibits a faster reduction in regret as $N$ increases, suggesting the advantages of the selective exploitation performed over a subset of contexts in each episode. This trend is also evident in the time-averaged regret plots, where AUCBBP consistently outperforms UCBBP, particularly during early learning phases.

In contrast, the baseline algorithm Epsilon-Greedy exhibits nearly constant time-averaged regret throughout the episodes, indicating that its cumulative regret grows linearly over time. A similar trend is observed in its context-averaged regret, which does not improve with increasing $N$.

\section{Conclusion}

We introduced the MCCB framework to model systems facing cascading feedback, heterogeneous per-click rewards, and multi-user episodes. MCCB captures essential structural features of real-world platforms.
We propose two algorithms to leverage these structures. In particular, AUCBBP achieves improved regret bounds by selectively exploring only a subset of user sessions, effectively reducing exploration overhead. This design enables the algorithm to benefit from \emph{economies of scale}: as the number of users per episode increases, the learner gains more information while incurring only marginal additional cost in regret. 
\bibliographystyle{informs2014}
\bibliography{references}

\clearpage
\begin{APPENDICES}

\section{Notation Summary}\label{append:notation}

\label{appendix:notation}
\begin{table}[h]
\centering
\begin{tabular}{ll}
\toprule
\textbf{Symbol} & \textbf{Description} \\
\midrule
$x_{t,n}$ & Context vector of user $n$ at episode $t$ \\
$k$ & Index of an arm (item) \\
$k^*$ & Oracle-best arm (with highest expected reward under true model) \\
$k^{\mathrm{ucb}}$ & Arm selected by the UCB-based policy \\
$k^{\mathrm{grd}}$ & Arm selected by the greedy policy (using current estimate) \\
$\theta$ & True model parameter vector \\
$\hat{\theta}_t$ & Estimated model parameter at episode $t$ \\
$z_{t,n,k}$ & Joint feature vector for user $n$ and arm $k$ at episode $t$ \\
$H$ & Interaction horizon (max steps per user) \\
$N$ & Number of users per episode \\
$T$ & Number of episodes \\
$T_0$ & Number of warm up episodes\\
$\widetilde{T}$ & Total number of backward-planned context-arm pulls \\
$A_t$ & Design matrix (regularized Gram matrix) at episode $t$ \\
$e_k$ & Reward of arm $k$ \\
$e_{\max}$ & Maximum possible reward across all arms, i.e., $\max_k e_k$ \\

$\beta_t$ & Confidence radius at episode $t$ \\
$k_{t,n,h}$ & Chosen arm for user $n$ in step $h$ of episode $t$ \\
$p_{\min}$ & Minimum probability for choosing an arm \\
$c_z$ & Upper bound on joint feature norm \\
$\kappa$ & Constant factor in confidence radius \\
$M_t$ & Number of actively explored contexts in episode $t$ \\
$Q^*(x,k)$ & True expected reward of choosing arm $k$ under context $x$ (under $\theta^*$) \\
$\widehat{Q}(x,k)$ & Predicted Q-value based on estimated model parameters \\
$w_{t,n,h}$ & Weighting factor used in design matrix update for context $(t,n,h)$ \\
$h^*_{t,n}$ & Step at which user $n$ in episode $t$ is absorbed (or $H$ if no reward) \\
\bottomrule
\end{tabular}
\end{table}

\section{Additional Assumptions}\label{append:assumption}

\begin{description}
  \item[(A4)] \textbf{Strictly positive absorbing probability with margin.}  
  We define $\sigma_{\min} := \sigma(-c_z) > 0$ and assume:
  \[
    \sigma_{\min} > \sup_{t > T_0}
    \frac{1}{4}\, \beta_t \sqrt{z^\top A_t^{-1} z}, \quad \forall z.
  \]
  This ensures that the predicted absorbing probability remains bounded away from zero.

  \item[(A5)] \textbf{Warm‑start contraction guarantee.}  
  There exists $T_0 \in \mathbb{N}$ such that for all $t > T_0$ and any feature vector $z$,
  \[
    \hat\sigma := \sigma(z^\top \hat\theta_t)
    \ge
    \sigma_{\min} - \frac14\, \beta_t \sqrt{z^\top A_t^{-1} z}.
  \]
\end{description}

\section{Proof of Propositions}\label{app:prop_proofs}
\begin{proof}{Proof of Proposition~\ref{prop:unimodal}}
Let $V := V^\star_{L-1}(x)$. Then we can write
\[
Q_L(k;x) = V + p_k\,d_k, \quad
p_k := f_k(x) \text{ (non-increasing),} \quad
d_k := e_k - V \text{ (non-decreasing)}.
\]
Since $V$ is constant in $k$, unimodality of $Q_L(k;x)$ reduces to that of $g_k := p_k d_k$.

We compute the first difference:
\[
g_{k+1} - g_k = (p_{k+1} - p_k)\,d_k + p_{k+1}(d_{k+1} - d_k).
\]
The first term is negative since $p_{k+1} < p_k$ and $d_k > 0$;  
the second term is positive since $p_{k+1} > 0$ and $d_{k+1} > d_k$.

As $k$ increases, the magnitude of the negative term increases
(because $p_k$ decreases), while the positive term grows more slowly
(since $d_k$ increases mildly and $p_k$ shrinks).  
Thus the sign of the difference changes at most once, implying that
$\{g_k\}$ increases, then decreases — establishing unimodality.
\end{proof}

\begin{proof}{Proof of Proposition~\ref{prop:monotone_arm_choice}}
Fix two arms $i < j$ and a context $x$. Define
\[
\Delta_{ij}(h) = Q(h,i;x) - Q(h,j;x)
= f_i(x)e_i - f_j(x)e_j + \bigl(f_j(x) - f_i(x)\bigr)\,V^\star_{h+1}(x).
\]
Since $f_i(x) \ge f_j(x)$, the coefficient of $V^\star_{h+1}(x)$ is non-positive.
Also, $V^\star_{h+1}(x)$ is non-increasing in $h$ (by backward induction),
so $\Delta_{ij}(h)$ is non-decreasing.
$\Delta_{ij}(h)$ is negative at some large $h$, it can only increase (or stay constant) as $h$ decreases.
Hence, arm $j$ can only be optimal before arm $i$ once — at most one index switch.
This "single crossing" property implies that the optimal arm index $k_h^\star(x)$ decreases (or stays the same) as $h$ decreases.
\end{proof}

\section{Preparatory Lemmas}\label{appendix:Lemmas}

This section presents all preparatory lemmas to ensure that the proofs of the two key theorems in Appendices C and D are presented clearly.

\begin{proof}{Proof of Lemma~\ref{lem:action-q-bound}}

Fix $(t, n, h)$ and let $x := x_{t,n}$ and fix arm $k \in [K]$.
Using the recursive definition of $Q$-values:
\[
\begin{aligned}
Q_h^{*}(x, k)
  &= \sigma(z_{t,n,k}^{\!\top} \theta)\, e_k
   + \bigl(1 - \sigma(z_{t,n,k}^{\!\top} \theta)\bigr) V_{h+1}^{*}(x), \\
\hat Q_h(x, k)
  &= \sigma(z_{t,n,k}^{\!\top} \hat\theta_t)\, e_k
   + \bigl(1 - \sigma(z_{t,n,k}^{\!\top} \hat\theta_t)\bigr) \hat V_{h+1}(x).
\end{aligned}
\]

Subtracting and regrouping:
\begin{flalign*}
Q_h^*(x,k) - \hat Q_h(x,k)
&= \bigl( \sigma(z^\top \theta) - \sigma(z^\top \hat\theta_t) \bigr) e_k 
 + \bigl( \sigma(z^\top \hat\theta_t) - \sigma(z^\top \theta) \bigr) V_{h+1}^*(x) \\
&\quad + \bigl(1 - \sigma(z^\top \hat\theta_t)\bigr)
        \bigl( V_{h+1}^*(x) - \hat V_{h+1}(x) \bigr).
\end{flalign*}

Taking absolute values and applying triangle inequality:
\begin{flalign*}
\bigl| Q_h^{*}(x,k) - \hat Q_h(x,k) \bigr|
\;\le\;& \bigl| \sigma(z^\top \theta) - \sigma(z^\top \hat\theta_t) \bigr|
       \bigl(e_k + V_{h+1}^*(x)\bigr) \\
&+ \bigl(1 - \sigma(z^\top \hat\theta_t)\bigr)
       \bigl| V_{h+1}^{*}(x) - \hat V_{h+1}(x) \bigr|.
\end{flalign*}

Substituting back \( z := z_{t,n,k} \), the claim follows.

\end{proof}

\bigskip

\begin{proof}{Proof of Lemma~\ref{lem:HfreeQbound}}
Let $x := x_{t,n}$ and define 
\[
\hat\sigma_{\min} := \min_{k \in [K]} \sigma(z_{t,n,k}^{\!\top} \hat\theta_t).
\]
This is strictly positive since $\sigma(\cdot) \in (0,1)$.

By Lemma~\ref{lem:action-q-bound}, for any arm $k \in [K]$ and any step $h$:
\[
\begin{aligned}
\bigl| Q_h^*(x,k) - \hat Q_h(x,k) \bigr|
\;\le\;
&\bigl| \sigma(z_{t,n,k}^{\!\top} \theta)
     - \sigma(z_{t,n,k}^{\!\top} \hat\theta_t) \bigr|
  \cdot (e_k + V_{h+1}^*(x)) \\
&+ \bigl(1 - \sigma(z_{t,n,k}^{\!\top} \hat\theta_t)\bigr)
  \cdot \bigl| V_{h+1}^*(x) - \hat V_{h+1}(x) \bigr|.
\end{aligned}
\]

Iterating this bound over $j = 0, \dots, H - h$, using $e_k \le e_{\max}$ and $V_{h+1}^*(x) \le V_{\max} \le e_{\max}$:
\[
\bigl| Q_h^*(x,k) - \hat Q_h(x,k) \bigr|
\;\le\;
\sum_{j = 0}^{H - h}
  (1 - \hat\sigma_{\min})^j
  (2 e_{\max})
  \cdot \max_{k}
  \bigl| \sigma(z_{t,n,k}^{\!\top} \theta)
       - \sigma(z_{t,n,k}^{\!\top} \hat\theta_t) \bigr|.
\]

Using the geometric sum bound:
\[
\sum_{j = 0}^{H - h} (1 - \hat\sigma_{\min})^j
\;\le\;
\frac{1}{\hat\sigma_{\min}},
\]
we conclude
\[
\bigl| Q_h^*(x,k) - \hat Q_h(x,k) \bigr|
\;\le\;
\frac{2 e_{\max}}{\hat\sigma_{\min}}
\cdot \max_{k}
  \bigl| \sigma(z_{t,n,k}^{\!\top} \theta)
       - \sigma(z_{t,n,k}^{\!\top} \hat\theta_t) \bigr|.
\]
\end{proof}

\bigskip

\begin{proof}{Proof of Lemma~\ref{lem:greedy-regret}}

From Lemma~\ref{lem:HfreeQbound}, under the logistic model we have:
\[
V_1^{\pi^*}(x_{t,n}) - V_1^{\pi_t}(x_{t,n}) 
   \;=\;
    Q_h^{*}(x_{t,n,h},k^{*})
    -Q_h^{*}(x_{t,n,h},k^{\mathrm{ucb}})
  \]
\[
\;\le\;
\underbrace{
    Q_h^{*}(x_{t,n,h},k^{*})
    -\hat Q_h(x_{t,n,h},k^{*})
}_{\text{(value deviation of }k^{*}\text{)}}
\;+\;
\underbrace{
    \hat Q_h(x_{t,n,h},k^{*})
    -\hat Q_h(x_{t,n,h},k^{\mathrm{grd}})
}_{\substack{\text{(estimated gap)} \\ \text{($\le 0$ by greedy choice)}}}
\;+\;
\underbrace{
    \hat Q_h(x_{t,n,h},k^{\mathrm{grd}})
    -Q_h^{*}(x_{t,n,h},k^{\mathrm{grd}})
}_{\text{(value deviation of }k^{\mathrm{grd}}\text{)}}.
\]
\[ \;\le\;
  \frac{4e_{\max}}{\hat\sigma_{\min}}\,
  \left| \sigma(z_{t,n,k^*}^\top \theta) - \sigma(z_{t,n,k^{\mathrm{grd}}}^\top \theta) \right|.
\]

Since $\mu(\cdot) = \sigma(\cdot)$ is $L$-Lipschitz with $L = \sup_u \sigma'(u) = \tfrac{1}{4}$, we apply:
\begin{align*}
  \left| \sigma(z_{t,n,k^*}^\top \theta) - \sigma(z_{t,n,k^{\mathrm{grd}}}^\top \theta) \right|
  &\le
  L \left| z_{t,n,k^*}^\top \theta - z_{t,n,k^{\mathrm{grd}}}^\top \theta \right| \\
  &=
  L \left| z_{t,n,k^*}^\top (\theta - \hat\theta_t)
          + z_{t,n,k^*}^\top \hat\theta_t
          - z_{t,n,k^{\mathrm{grd}}}^\top \hat\theta_t
          + z_{t,n,k^{\mathrm{grd}}}^\top (\hat\theta_t - \theta)
    \right|.
\end{align*}

By definition of the greedy arm:
\[
  z_{t,n,k^*}^\top \hat\theta_t
  \;\le\;
  z_{t,n,k^{\mathrm{grd}}}^\top \hat\theta_t,
\]
so the difference $z_{t,n,k^*}^\top \hat\theta_t - z_{t,n,k^{\mathrm{grd}}}^\top \hat\theta_t \le 0$.

Thus,
\[
  \left| \mu(z_{t,n,k^*}^\top \theta) - \mu(z_{t,n,k^{\mathrm{grd}}}^\top \theta) \right|
  \le
  L (
    \left| z_{t,n,k^*}^\top(\theta - \hat\theta_t) \right|
    +
    \left| z_{t,n,k^{\mathrm{grd}}}^\top(\hat\theta_t - \theta) \right|
  ).
\]

Substituting into the value bound gives the result.
\end{proof}
\bigskip

\begin{lemma}[Warm-up sample complexity for estimation accuracy]
\label{lem:B0}
Let
\[
  \varepsilon_{\max}
  := \frac{\Delta_0}{c_z\,(e_{\max})},
\]
and define
\[
  T_{\mathrm{0}}
  := \Bigl\lceil
       \frac{8(d+K)(e_{\max})^{2}
             \log\!\bigl(1/\delta\bigr)}
            {K\,w_{\min}p_{\min}\,c_z^{2}\,\Delta_0^{2}}
      \Bigr\rceil.
\]
Since our algorithm performs \(T_{\mathrm{0}}\) episodes of uniform exploration in which every arm is selected exactly once per user slot, then with probability at least \(1 - \delta\),
\[
  \bigl\|\hat\theta_{t_0} - \theta\bigr\|_2 \le \varepsilon_{\max},
  \qquad
  \lambda_{\min}(A_{t_0}) \ge
  \frac{(d+K)\log(1/\delta)}{\varepsilon_{\max}^{2}}.
\]
Consequently, the estimation-accuracy condition of Lemma~\ref{lem:B2} and the requirement $\varepsilon < \Delta_0 / (2L_Q)$ are both satisfied.
\end{lemma}

\bigskip

\begin{lemma}[Selection–probability lower bound under accurate estimation and margin]
\label{lem:B2}
Let $\varepsilon > 0$, $\Delta > 0$, $\gamma \in (0,1]$, $T_0 \in \mathbb{N}$, and $c_z > 0$.
Assume that the estimation error satisfies $\|\hat\theta_t - \theta\|_2 \le \varepsilon$ for all $t \ge T_0$, the margin condition (A3) holds with parameters $(\Delta, \gamma)$, and (A1) holds the feature norm is bounded as $\|z_j(x)\|_2 \le c_z$ for all $x$ and arms $j$.

Define the Lipschitz constant
\[
  L_Q := \tfrac{1}{2} \, c_z \, (e_{\max} ).
\]
If $\varepsilon < \Delta / (2L_Q)$, then for all $t \ge t_0$,
\[
  \tilde p_{\min}(t)
  := \inf_{h \in [H]} \Pr_{x \sim \mathcal{D}} \left[
      \arg\max_j \hat Q_h(x, j)
      = \arg\max_j Q^*_h(x, j)
  \right]
  \;\ge\; \gamma.
\]
\end{lemma}

\begin{proof}{Proof of Lemma~\ref{lem:B2}}
From (A1), we have $\|z_j(x)\|_2 \le c_z$, so for any arm $j$,
\[
  \left| \hat Q_h(x, j) - Q^*_h(x, j) \right|
  \;\le\; L_Q \, \varepsilon.
\]

By (A3), with probability $\gamma$ over $x \sim \mathcal{D}$,
\[
  Q^*_h(x, k^*) - \max_{j \ne k^*} Q^*_h(x, j) \ge \Delta.
\]
In such cases, for any $j \ne k^*(x)$,
\[
\hat Q_h(x, k^*) - \hat Q_h(x, j)
  \;\ge\; \Delta - 2L_Q \varepsilon
  \;>\; 0,
\]
where the last inequality uses $\varepsilon < \Delta / (2L_Q)$.

Thus, the arm selected using $\hat\theta_t$ coincides with the optimal arm under $\theta$.
So for all $h \in [H]$ and $t \ge T_0$,
\[
  \Pr_{x \sim \mathcal{D}} \left[
    \arg\max_j \hat Q_h(x, j)
    = \arg\max_j Q_h^*(x, j)
  \right]
  \;\ge\; \gamma.
\]
Taking the infimum over $h$ yields the claim.
\end{proof}

\bigskip

\begin{lemma}[Positive selection probability for every arm under positive margin]
\label{lem:B3}
Assume that the estimation error satisfies $\|\hat\theta_t - \theta\|_2 < \Delta_0 / (2L_Q)$ for all $t \ge T_0$, and that Assumption~(A3) (arm-wise positive margin) holds.

Then for every $t \ge T_0$, every step $h \in [H]$, and every arm $k \in [K]$,
\[
  \Pr_{x \sim \mathcal{D}} \left[ k_{t,n,h} = k \right]
  \;\ge\;
  \gamma_0.
\]
\end{lemma}

\begin{proof}{Proof of Lemma~\ref{lem:B3}}
Fix an arm $k$ and step $h$. Define the event
\[
  A := \left\{
    k = \arg\max_{j \in [K]} Q^*_h(x, j)
    \;\text{ and }\;
    Q^*_h(x, k) - \max_{j \ne k} Q^*_h(x, j) \ge \Delta_0
  \right\}.
\]
By Assumption~A3, we have $\Pr_{x \sim \mathcal{D}}[A] \ge \gamma_0$.

On event $A$, and under the condition $\|\hat\theta_t - \theta\|_2 < \Delta_0 / (2L_Q)$, we have by Lipschitz continuity of $Q_h$:
\[
  \hat Q_h(x, k)
  - \max_{j \ne k} \hat Q_h(x, j)
  \;\ge\;
  \Delta_0 - 2L_Q \, \|\hat\theta_t - \theta\|_2
  \;>\; 0,
\]
so the greedy selection also chooses arm $k$, i.e., $\hat{k} = k$.

Therefore,
\[
  \Pr[k_{t,n,h} = k]
  \;\ge\;
  \Pr[A]
  \;\ge\;
  \gamma_0.
\]
\end{proof}

\bigskip

\begin{definition}
We define per-arm minimum selection probability
\[
  p_{\min}
  := \inf_{h \in [H],\, k \in [K]}
      \Pr_{x \sim \mathcal{D}} \left[ k_{t,n,h} = k \right].
\] 
By Lemma~\ref{lem:B3}, we have \( p_{\min} \ge \gamma_0 > 0 \) for all \( t \ge T_0 \).
\end{definition}

\bigskip

\begin{lemma}[Lipschitz Property of the Link Function]
\label{lem:sigmoid-lipschitz}
The logistic sigmoid function
\[
\sigma(u) = \frac{1}{1 + e^{-u}}
\]
is $1/4$-Lipschitz, that is,
\[
| \sigma(u_1) - \sigma(u_2) | \leq \frac{1}{4} \cdot | u_1 - u_2 |, \quad \forall u_1, u_2 \in \mathbb{R}.
\]
\end{lemma}

\begin{proof}{Proof of Lemma~\ref{lem:sigmoid-lipschitz}}
The derivative of $\sigma(u)$ is
\[
\sigma'(u) = \sigma(u) (1 - \sigma(u)).
\]
The maximum of $\sigma'(u)$ is attained at $\sigma(u) = 1/2$, giving
\[
\max_{u \in \mathbb{R}} \sigma'(u) = \frac{1}{4}.
\]
Thus, by the mean value theorem, for any $u_1, u_2 \in \mathbb{R}$,
\[
| \sigma(u_1) - \sigma(u_2) | \leq \max_{u \in \mathbb{R}} \sigma'(u) \cdot | u_1 - u_2 | = \frac{1}{4} | u_1 - u_2 |.
\]
\end{proof}
\bigskip
\begin{lemma}[Martingale Difference Property of $\eta_{t,n,h}$]
\label{lem:mds}
The noise terms $\eta_{t,n,h}$ defined by
\[
\eta_{t,n,h} = r_{t,n,h} - e_{k_{t,n,h}} \cdot \sigma\left( z_{t,n,k_{t,n,h}}^\top \theta \right)
\]
satisfy the following property:
\[
\mathbb{E} \left[ \eta_{t,n,h} \ \middle|\ \mathcal{F}_{t,n,h-1} \right] = 0.
\]
Thus, $\{ \eta_{t,n,h}, \mathcal{F}_{t,n,h} \}$ forms a martingale difference sequence.
\end{lemma}
\begin{proof}{Proof of Lemma~\ref{lem:mds}}
Given $\mathcal{F}_{t,n,h-1}$, the arm choice $k_{t,n,h}$ and feature $z_{t,n,k_{t,n,h}}$ are known, while the reward $r_{t,n,h}$ is generated as
\[
r_{t,n,h} \ \big| \ \mathcal{F}_{t,n,h-1} \sim \text{Bernoulli}\left( \sigma\left( z_{t,n,k_{t,n,h}}^\top \theta \right) \right) \cdot e_{k_{t,n,h}}.
\]

Thus, the conditional expectation of $r_{t,n,h}$ given $\mathcal{F}_{t,n,h-1}$ is
\[
\mathbb{E} \left[ r_{t,n,h} \ \middle|\ \mathcal{F}_{t,n,h-1} \right]
= e_{k_{t,n,h}} \cdot \sigma\left( z_{t,n,k_{t,n,h}}^\top \theta \right).
\]

Substituting this into the definition of $\eta_{t,n,h}$, we obtain
\[
\mathbb{E} \left[ \eta_{t,n,h} \ \middle|\ \mathcal{F}_{t,n,h-1} \right]
\]
\[
= \mathbb{E} \left[ r_{t,n,h} \ \middle|\ \mathcal{F}_{t,n,h-1} \right] 
- e_{k_{t,n,h}} \cdot \sigma\left( z_{t,n,k_{t,n,h}}^\top \theta \right)
\]
\[
= e_{k_{t,n,h}} \cdot \sigma\left( z_{t,n,k_{t,n,h}}^\top \theta \right)
- e_{k_{t,n,h}} \cdot \sigma\left( z_{t,n,k_{t,n,h}}^\top \theta \right) = 0.
\]

This completes the proof.
\end{proof}
\bigskip
\begin{lemma}[Sub-Gaussianity of $\eta_{t,n,h}$]
\label{lem:subgaussian-eta}
For each $(t,n,h)$, conditioned on $\mathcal{F}_{t,n,h-1}$, the noise term $\eta_{t,n,h}$ is $e_{k_{t,n,h}}$-sub-Gaussian. That is, for all $\lambda \in \mathbb{R}$,
\[
\mathbb{E} \left[ \exp\left( \lambda \cdot \eta_{t,n,h} \right) \ \middle|\ \mathcal{F}_{t,n,h-1} \right] 
\leq \exp\left( \frac{ \lambda^2 \cdot e_{k_{t,n,h}}^2 }{8} \right).
\]
\end{lemma}

\begin{proof}{Proof of Lemma ~\ref{lem:subgaussian-eta}}
The noise $\eta_{t,n,h}$ is the centered reward:
\[
\eta_{t,n,h} = r_{t,n,h} - \mathbb{E} \left[ r_{t,n,h} \ \middle|\ \mathcal{F}_{t,n,h-1} \right].
\]
Since $r_{t,n,h} \in \{ 0, e_{k_{t,n,h}} \}$ is a bounded random variable, it follows from Hoeffding’s lemma that $\eta_{t,n,h}$ is $e_{k_{t,n,h}}$-sub-Gaussian with variance proxy $e_{k_{t,n,h}}^2 / 4$.

Therefore,
\[
\mathbb{E} \left[ \exp\left( \lambda \cdot \eta_{t,n,h} \right) \ \middle|\ \mathcal{F}_{t,n,h-1} \right] 
\leq \exp\left( \frac{ \lambda^2 \cdot e_{k_{t,n,h}}^2 }{8} \right).
\]
\end{proof}
\paragraph{Flattened notation.}
We flatten all observations $(t,n,h)$ into a single index $i = 1, 2, \dots, \widetilde{T}$, where
\[
\widetilde{t} = \sum_{t=1}^t \sum_{n=1}^N h^*_{t,n}.
\]
Let $(t(i), n(i), h(i))$ denote the episode, context, and step for index $i$. Define
\[
m_i := \sqrt{w_{t(i), n(i), h(i)}} \cdot z_{t(i), n(i), k_{t(i), n(i), h(i)}}, \quad
\eta_i := \eta_{t(i), n(i), h(i)}.
\]
Let the design matrices be
\[
M_{\widetilde{t}} := \sum_{i=1}^{\widetilde{t}} m_i m_i^\top,
\quad
A_{\widetilde{t}} := \lambda I + M_{\widetilde{t}},
\quad
\xi_{\widetilde{t}} := \sum_{i=1}^{\widetilde{t}} m_i \eta_i.
\]

\medskip
\noindent
\textbf{Weight bound.}  Because the logistic sigmoid satisfies
\(\displaystyle\max_{u\in\mathbb{R}}\sigma'(u)=\tfrac14\),
we have
\[
0 \;\le\; w_{t(i), n(i), h(i)} \;\le\; w_{\max} := \frac{1}{4}.
\]
(Using \(\sqrt{w_i}\) in \(m_i\) ensures the Fisher‑information–style design matrix while keeping \(\|m_i\|\le\sqrt{w_{\max}}\,\|z\|\).)

\bigskip
\noindent
\begin{lemma}[Self‑Normalized Confidence Bound]
\label{lem:selfnorm}
Under {\normalfont(A1)}, Lemma~\ref{lem:sigmoid-lipschitz},
Lemma~\ref{lem:mds}, and Lemma~\ref{lem:subgaussian-eta},
for any $0 < \delta < 1/e$, with probability at least $1-\delta$,
\[
  \bigl\|\xi_{\widetilde t}\bigr\|_{A_{\widetilde t}^{-1}}
  \;\le\;
  \beta_{\widetilde t},
\]
where
\[
  \beta_{\widetilde t}
  :=\;
  \kappa\,e_{max}\,
  \sqrt{\,2(d+K)\,
         \log\widetilde t\,
         \log\!\Bigl(\tfrac{d+K}{\delta}\Bigr)},
  \quad
  e_{max} = \max_k e_k, \quad
  \kappa
  :=\sqrt{\;
        3 + 2\log\!\Bigl(1+\tfrac{w_{\max}c_z^{\,2}}{2\lambda}\Bigr)},
  \quad
  w_{\max}= \tfrac14 .
\]
\end{lemma}

\begin{proof}{Proof of Lemma ~\ref{lem:selfnorm}}
Fix $x\in\mathbb R^{d+K}$ and consider the scalar martingale
\[
  S_j := x^{\!\top}\!\sum_{i=1}^{j} m_i\,\eta_i,
  \qquad j=0,1,\dots,\widetilde T.
\]
By Lemmas~\ref{lem:mds}–\ref{lem:subgaussian-eta},
the increments $x^{\!\top}m_j\eta_j$ are conditionally
mean‑zero and $(R\,|x^{\!\top}m_j|)$‑sub‑Gaussian.

Applying the exponential tail bound for vector-valued martingales
(adapted from Lemma~1 of \cite{Filippi2010}), we obtain, for any $0<\delta<1/e$,
\[
  \Pr\!\Bigl\{
    |x^{\!\top}\xi_{\widetilde t}|
    \;\ge\;
    \kappa e_{max}
      \sqrt{\,
        \|x\|_{M_{\widetilde t}}^{2}\;
        2\log\widetilde t\,
        \log\!\bigl(\tfrac{d+K}{\delta}\bigr)}
  \Bigr\}
  \;\le\;
  \tfrac{\delta}{d+K}.
\]

Take an $\varepsilon$‑net of the unit ball in the
$M_{\widetilde T}$‑norm and union‑bound over at most
$\bigl(1+2/\varepsilon\bigr)^{d+K}$ points with
$\varepsilon=(d+K)^{-1}$ to obtain
\[
  \sup_{\|x\|_{M_{\widetilde t}}=1} |x^{\!\top}\xi_{\widetilde t}|
  \;\le\;
  \beta_{\widetilde t}
  \qquad\text{w.p. }1-\delta.
\]
Hence
$\|\xi_{\widetilde t}\|_{M_{\widetilde t}^{-1}}\le\beta_{\widetilde t}$.
Because $A_{\widetilde t}=\lambda I+M_{\widetilde t}\succeq M_{\widetilde t}$,
we have $A_{\widetilde t}^{-1}\preceq M_{\widetilde t}^{-1}$, so
$\|\xi_{\widetilde t}\|_{A_{\widetilde t}^{-1}}\le\beta_{\widetilde t}$.
\end{proof}

\bigskip

\begin{lemma}[Prediction–Error Bound]
\label{lem:pred-ci}
Under {\normalfont(A1)} and the high-probability event of
Lemma~\ref{lem:selfnorm}, for every $t \ge d + 1$, any arm feature
$z \in \mathbb{R}^{d+K}$ with $m := \sqrt{w} \, z$ and
$w = \sigma'(z^\top \hat\theta_{t-1})$, and any $0 < \delta < 1/e$, the
prediction error is bounded as:
\[
\bigl| \sigma(z^\top \theta) - \sigma(z^\top \hat\theta_t) \bigr|
\;\le\;
\frac{2 k_\mu \, \kappa \, e_{max}}{c_\mu}
\cdot
\|m\|_{A_t^{-1}}
\cdot
\sqrt{\,2(d+K) \log\widetilde T\cdot \log\!\left( \tfrac{d+K}{\delta} \right)},
\]
where
\[
\begin{aligned}
k_\mu &= \tfrac{1}{4}, \qquad
c_\mu := \min_{|u| \le B} \sigma'(u) > 0, \qquad
e_{max} = \max_k e_k, \quad
\kappa := \sqrt{3 + 2 \log\!\left( 1 + \tfrac{w_{\max} c_z^2}{2 \lambda} \right)},
\qquad
w_{\max} = \tfrac{1}{4}.
\end{aligned}
\]
\end{lemma}

\emph{Remark.}
This result follows from Proposition~1 of \cite{Filippi2010}, with notational alignment to our setting:
$c_m = \sqrt{w_{\max}} \, c_z$, $e_{max} = \max_k e_k$ and $\lambda_0 = \lambda$.
All required conditions—bounded and sub-Gaussian features, adaptivity, and regularization—are satisfied in our setting; see Lemmas~\ref{lem:sigmoid-lipschitz}–\ref{lem:selfnorm}.
\bigskip
\begin{lemma}[Single–step UCBgap]\label{lem:ucbbp_gap}
For episode $t\le T$, context $n\le N$, horizon $h\le H$, let
\[
  k_{t,n,h}
  :=\arg\max_{k\in[K]}
     \Bigl\{
       \widehat Q(x_{t,n,h},k)+
       \beta_{t-1}\sqrt{z_{t,n,k}^{\!\top}A_{t-1}^{-1}z_{t,n,k}}
     \Bigr\}.
\]
Then, for any comparator arm $k\in[K]$,
\[
  \bigl|
     \widehat Q(x_{t,n,h},k)-
     \widehat Q(x_{t,n,h},k_{t,n,h})
  \bigr|
  \;\le\;
  \beta_{t-1}\Bigl(
     \|z_{t,n,k}\|_{A_{t-1}^{-1}}
     +\|z_{t,n,k_{t,n,h}}\|_{A_{t-1}^{-1}}
  \Bigr).
\]
\end{lemma}

\begin{proof}{Proof of Lemma~\ref{lem:ucbbp_gap}}
Let $s_{t-1}(k):=\|z_{t,n,k}\|_{A_{t-1}^{-1}}$.  
Because $k_{t,n,h}$ maximizes the UCB index,
\[
  \widehat Q(x,k)+\beta_{t-1}s_{t-1}(k^*)
  \;\le\;
  \widehat Q(x,k_{t,n,h})+\beta_{t-1}s_{t-1}(k_{t,n,h}).
\]
Re-arranging gives the one-sided bound
\[
  \widehat Q(x,k)-\widehat Q(x,k_{t,n,h})
  \;\le\;\beta_{t-1}\bigl(s_{t-1}(k_{t,n,h})-s_{t-1}(k)\bigr).
\]
Swapping the roles of $k$ and $k_{t,n,h}$ yields the opposite
sign inequality.  Combining the two and using
\(|a-b|\le|a|+|b|\) with non-negative radii
(\(s_{t-1}(\,\cdot\,)\ge0\)) gives
\[
  \bigl|
    \widehat Q(x,k)-\widehat Q(x,k_{t,n,h})
  \bigr|
  \;\le\;
  \beta_{t-1}\bigl(s_{t-1}(k)+s_{t-1}(k_{t,n,h})\bigr),
\]
completing the proof.
\end{proof}
\bigskip
\begin{lemma}[Linear Growth of $\lambda_{\min}(A_t)$ under Uniform Arm Coverage]
\label{lem:eigen-linear}
Under Assumptions~\textnormal{(A1)}–\textnormal{(A3)},  
suppose each full feature vector \(z_{t,n,k} = [\,x_{t,n};\, e_k\,]\) satisfies the global bound
\[
  \|z_{t,n,k}\|_2 \;\le\; c_z
\]
and the Fisher weights lie in \(w_{t,n,h} \in [w_{\min}, w_{\max}] \subset (0, \tfrac{1}{4}]\).  
For any \(0 < \delta < 1\), with probability at least \(1 - \delta\), the following bound holds for all \(t > T_0\):
\[
  \lambda_{\min}(A_{\widetilde{t}})
  \;\ge\;
  \lambda + \tfrac{1}{2} w_{\min} p_{\min} \widetilde t.
\]
\end{lemma}

\begin{proof}{Proof of Lemma~\ref{lem:eigen-linear}}
Let \(X_i := w_i z_i z_i^{\!\top}\) denote the $i$-th flattened outer product.
By (A.6), the arm indicator satisfies
\[
  \mathbb{E}[e_{k(i)} e_{k(i)}^{\!\top} \mid \mathcal{F}_{i-1}]
  \;\succeq\;
  p_{\min} I_K,
\qquad \text{and} \qquad
  \mathbb{E}[x_i x_i^{\!\top}] = I_d.
\]
Thus, we have
\[
  \mathbb{E}[X_i \mid \mathcal{F}_{i-1}]
  \;\succeq\;
  w_{\min}
  \left(
    I_d \;\oplus\; p_{\min} I_K
  \right)
  \;\succeq\;
  w_{\min} p_{\min} I_{d+K}.
\]
Taking expectation and summing gives
\[
  \mathbb{E}\!\left[\sum_{i = 1}^{\widetilde t} X_i\right]
  \;\succeq\;
  w_{\min} p_{\min} \widetilde t\, I,
\]
so the minimum eigenvalue is at least \(w_{\min} p_{\min} \widetilde t\).

Since each \(X_i\) is PSD and satisfies
\(
  \lambda_{\max}(X_i)
  \le w_{\max} c_z^2,
\)
Applying the matrix Chernoff inequality for predictable sequences ~\cite{Tropp_2011},
we obtain
\[
  \Pr\!\left\{
    \lambda_{\min}\left(\sum_{i=1}^{\widetilde t} X_i\right)
    \le \tfrac{1}{2} w_{\min} p_{\min} \widetilde t
  \right\}
\le
  (d + K)\, \exp\!\left(
    -\tfrac{w_{\min} p_{\min}}{2 w_{\max} c_z^2}
     \cdot \widetilde t
  \right)
  < \delta,
\]
provided that
\[
  \widetilde t
  \;\ge\;
  \tfrac{2 w_{\max} c_z^2}{w_{\min} p_{\min}}
  \log\!\left( \tfrac{d + K}{\delta} \right).
\]
On this high-probability event, we have
\[
  \sum_{i = 1}^{\widetilde t} X_i
  \;\succeq\;
  \tfrac{1}{2} w_{\min} p_{\min} \widetilde t\, I,
\]
and adding the ridge term \(\lambda I\) completes the proof.
\end{proof}
\bigskip

\begin{lemma}[Uniform‐Coverage Bound for the Max‑Arm Radius]
\label{lem:max_radius_uniform}
Assume \textnormal{(A1)}–\textnormal{(A3)} and let  
\(A_0=\lambda I\) and  
\(A_{\widetilde{t}}=\lambda I+\sum_{i=1}^{\widetilde{t}}w_i z_i z_i^{\!\top}\).
Suppose \(\lVert z_{t,n,k}\rVert_2\le c_z\) and \(c_\mu\le w_i\le\tfrac14\).
Assume that, with probability at least \(1-\delta\), the coverage condition
\[
  \lambda_{\min}(A_{\widetilde{t}})\;\ge\;\lambda+\tfrac12\,w_{\min}p_{\min}\,{\widetilde{t}}
  \quad\text{holds for all }\widetilde{t}\ge0.
\]
Then, on this high‑probability event, for every flatten episode \(\widetilde T\ge0\),
\[
  \sum_{t=1}^{\widetilde T}
    \min\!\bigl\{
      \max_{k\in[K]}\lVert z_{t,n,k}\rVert_{A_{t-1}^{-1}}^{2},
      1
    \bigr\}
  \;\le\;
  \frac{4c_{z}^{2}}{w_{\min}p_{\min}}\,
  \log\!\Bigl(
    1+\frac{w_{\min}p_{\min}}{2\lambda}\,\widetilde T
  \Bigr).
\]
\end{lemma}

\begin{proof}{Proof of Lemma~\ref{lem:max_radius_uniform}}
For fixed \(t\ge1\),
\[
  \max_{k}\lVert z_{t,n,k}\rVert_{A_{t-1}^{-1}}^{2}
  \;\le\;
  \frac{c_{z}^{2}}
       {\lambda+\tfrac12 w_{\min}p_{\min}(t-1)}.
\]
Because the right‑hand side is at most \(1\) when \(\lambda\ge c_z^{2}\),
the inequality \(u\le 2\ln(1+u)\) for \(u\in[0,1]\) yields
\[
  \min\!\bigl\{
      \max_{k}\lVert z_{t,n,k}\rVert_{A_{t-1}^{-1}}^{2},
      1
    \bigr\}
  \;\le\;
  \frac{2c_{z}^{2}}
       {\lambda+\tfrac12 w_{\min}p_{\min}(t-1)}.
\]
Summing over \(t=1,\dots,\widetilde T\),
\[
  \sum_{t=1}^{\widetilde T}
    \min\!\bigl\{
      \max_{k}\lVert z_{t,n,k}\rVert_{A_{t-1}^{-1}}^{2},
      1
    \bigr\}
  \;\le\;
  2c_{z}^{2}
  \sum_{t=1}^{\widetilde T}
    \frac{1}
         {\lambda+\tfrac12 w_{\min}p_{\min}(t-1)}.
\]
For \(a>0\) and \(b>0\),
\(\sum_{t=1}^{\widetilde T}\frac{1}{a+b(t-1)}
   \le\frac{2}{b}\ln\!\bigl(1+\tfrac{b}{2a}\widetilde T\bigr)\).
Taking \(a=\lambda\) and \(b=\tfrac12 w_{\min}p_{\min}\) gives
\[
  2c_{z}^{2}
  \sum_{t=1}^{\widetilde T}
    \frac{1}
         {\lambda+\tfrac12 w_{\min}p_{\min}(t-1)}
  \;\le\;
  \frac{4c_{z}^{2}}{w_{\min}p_{\min}}\,
  \log\!\Bigl(
    1+\frac{w_{\min}p_{\min}}{2\lambda}\,\widetilde T
  \Bigr),
\]
which completes the proof.
\end{proof}
\bigskip
\begin{lemma}[Cumulative single-step UCB gap]
\label{lem:cum_ucb_gap}
Assume \textnormal{(A1)}–\textnormal{(A5)} and our warm-up horizon $T_{0}$ guarantee $\lambda\!\ge\!c_{z}^{2}$ at $t=T_{0}$.               
Let $i_{0}$ be the flattened index that corresponds to the first
triple $(t,n,h)$ with $t=T_{0}$. Work on the $(1-\delta)$ event of
Lemma~\ref{lem:max_radius_uniform}.
Then
\[
  \sum_{i={\widetilde T_0}+1} ^{\widetilde T}
    \bigl|
      \hat Q_{h(i)}(x_i,k_i^{*})
      -\hat Q_{h(i)}(x_i,k_i^{\mathrm{ucb}})
    \bigr|
  \;\le\;
  2\,
  \beta_{\widetilde T}\,
  \sqrt{\frac{4c_{z}^{2}}{w_{\min}p_{\min}}\,
        \widetilde T\,
        \log\!\Bigl(
          1+\frac{w_{\min}p_{\min}}{2\lambda}\,\widetilde T
        \Bigr)}.
\]
\end{lemma}

\begin{proof}{Proof of Lemma~\ref{lem:cum_ucb_gap}}

For every $i\!\ge\!{\widetilde T_0}+1$, Lemma~\ref{lem:ucbbp_gap} gives
\[
  \bigl|
    \hat Q_{h(i)}(x_i,k_i^{*})
    -\hat Q_{h(i)}(x_i,k_i^{\mathrm{ucb}})
  \bigr|
  \le
  2\,
  \beta_{i-1}\,
  \max_{k\in[K]}\|z_{t(i),n(i),k}\|_{A_{i-1}^{-1}}
  \le
  2\,
  \beta_{\widetilde T}\,
  \max_{k}\|z_{t(i),n(i),k}\|_{A_{i-1}^{-1}}\!,
\]
because $\beta_{i-1}\le\beta_{\widetilde T}$.

With $\lambda\ge c_{z}^{2}$ and the coverage condition
$\lambda_{\min}(A_{i-1})
 \ge\lambda$ already true at $i={\widetilde T_0}+1$,
\[
  \max_{k}\|z_{t(i),n(i),k}\|_{A_{i-1}^{-1}}^{2}
  \;\le\;
  \frac{c_{z}^{2}}{\lambda}
  \;\le\;1
  \quad\forall\,i\ge i_{0}.
\]
Hence the truncation
${\min\{\max_{k}\|z\|_{A^{-1}}^{2},\,1\}}$ in
Lemma~\ref{lem:max_radius_uniform} is always inactive
\emph{after} the warm-up.

Summing over $i={\widetilde T_0}+1,\dots,\widetilde T$ and applying
Cauchy–Schwarz,
\[
\sum_{i={\widetilde T_0}+1}^{\widetilde T}
  \max_{k}\|z_{t(i),n(i),k}\|_{A_{i-1}^{-1}}
\le
\sqrt{\widetilde T^{\,+}}\,
\Bigl(
  \sum_{i=i_{0}}^{\widetilde T}
    \max_{k}\|z_{t(i),n(i),k}\|_{A_{i-1}^{-1}}^{2}
\Bigr)^{1/2}.
\]
Applying Lemma~\ref{lem:max_radius_uniform} to the squared sum with
$\widetilde T^{\,+}$ instead of $\widetilde T$ yields
\[
  \sum_{i=i_{0}}^{\widetilde T}
    \max_{k}\|z_{t(i),n(i),k}\|_{A_{i-1}^{-1}}^{2}
  \le
  \frac{4c_{z}^{2}}{w_{\min}p_{\min}}\,
  \log\!\Bigl(
     1+\tfrac{w_{\min}p_{\min}}{2\lambda}\,\widetilde T
  \Bigr).
\]
Multiplying the outer constants finishes the proof.
\end{proof}

\bigskip
\begin{lemma}[ Cumulative Q-value Deviation]
\label{lem:Q-value-stability}
On the joint event where
Lemmas~\ref{lem:selfnorm},
\ref{lem:pred-ci},
\ref{lem:HfreeQbound},
and~\ref{lem:max_radius_uniform} all hold,
define
\[
C_t := \frac{2e_{\max}k_\mu}{c_\mu\hat\sigma_{\min}} \cdot \beta_{\widetilde T},
\quad
g_{t,n} := \min\left\{
  \max_{k\in[K]} \|z_{t,n,k}\|_{A_{\widetilde T}^{-1}}^2,\;
  1
\right\}.
\]
Then the total Q-value deviation is bounded as
\[
\sum_{t=T_0+1}^{T} \sum_{n=1}^{N}
  \bigl| Q_1^{*}(x_{t,n}, k) - \hat Q_1(x_{t,n}, k) \bigr|
\;\le\;
\frac{2e_{\max}k_\mu}{c_\mu\hat\sigma_{\min}} \cdot
\beta_{\widetilde T} \cdot
\sqrt{NT} \cdot
\sqrt{
  \frac{4c_z^2}{w_{\min}p_{\min}} \cdot
  \log\!\left(
    1 + \frac{w_{\min}p_{\min}}{2\lambda} \cdot \widetilde T
  \right)
}.
\]
\end{lemma}

\begin{proof}{Proof of Lemma~\ref{lem:Q-value-stability}}
By Lemma~\ref{lem:pred-ci} and the bound on $w_{i(k)} \le \tfrac14$, we have
\[
\bigl| Q_1^{*}(x_{t,n}, k) - \hat Q_1(x_{t,n}, k) \bigr|
\le
\frac{2e_{\max}}{\hat\sigma_{\min}} \cdot
\max_{k\in[K]}
\left| \sigma(z_{t,n,k}^{\top} \theta) - \sigma(z_{t,n,k}^{\top} \hat\theta_t) \right|
\le
\frac{2e_{\max}k_\mu}{c_\mu\hat\sigma_{\min}} \cdot
\beta_{\widetilde T} \cdot
\min\left\{ \max_{k\in[K]} \|z_{t,n,k}\|_{A_{\widetilde T}^{-1}},\, 1 \right\}.
\]
Letting $C_t$ and $g_{t,n}$ be defined as above, we have
\[
\bigl| Q_1^{*}(x_{t,n}, k) - \hat Q_1(x_{t,n}, k) \bigr| \le C_t \sqrt{g_{t,n}}.
\]
Applying Cauchy–Schwarz,
\[
\sum_{t=T_0+1}^{T} \sum_{n=1}^{N} C_t \sqrt{g_{t,n}}
\le
\left( \sum_{t=T_0+1}^{T} \sum_{n=1}^{N} C_t^2 \right)^{1/2}
\left( \sum_{t=T_0+1}^{T} \sum_{n=1}^{N} g_{t,n} \right)^{1/2}.
\]

Since $C_t$ is constant in $n$,
\[
\sum_{t,n} C_t^2
= N \sum_{t=T_0+1}^{T} C_t^2
\le N T \left( \frac{2e_{\max}k_\mu}{c_\mu\hat\sigma_{\min}} \right)^2 \beta_{\widetilde T}^2.
\]

From Lemma~\ref{lem:max_radius_uniform}, we have
\[
\sum_{t=T_0+1}^{T} \sum_{n=1}^{N} g_{t,n}
\le
\frac{4c_z^2}{w_{\min}p_{\min}} \cdot
\log\!\left(
  1 + \frac{w_{\min}p_{\min}}{2\lambda} \cdot \widetilde T
\right).
\]

Combining the two bounds yields the result.
\end{proof}
\bigskip
\begin{lemma}[Parameter estimation error bound]
\label{lem:param-error}
Let assumptions \textnormal{(A1)}–\textnormal{(A3)} hold. Then, with probability at least $1 - \delta$, for all $t \ge T_0$:
\[
\|\hat\theta_t - \theta\|_{A_{\widetilde t}}
\;\le\;
\frac{\beta_{\widetilde t}}{c_\mu},
\qquad
\|\hat\theta_t - \theta\|_2
\;\le\;
\frac{ \beta_{\widetilde t} }
     { c_\mu \sqrt{ \lambda + \tfrac{1}{2} w_{\min} p_{\min} \widetilde t } }.
\tag{4}\label{eq:param-error-l2}
\]

\noindent
\textnormal{where} $c_\mu = \min_{|u| \le B} \sigma'(u)$, $k_\mu = \tfrac{1}{4}$, $w_{\min} = \min_i w_i$, $w_{\max} = \max_i w_i$, and
\[
A_{\widetilde t} = \lambda I + \sum_{i=1}^{\widetilde t} w_i z_i z_i^\top,
\qquad
\beta_{\widetilde t}
= \kappa\, e_{\max} c_z
\sqrt{2(d+K)\log \widetilde t \cdot \log\!\left( \tfrac{d+K}{\delta} \right)},
\qquad
\kappa = \sqrt{ 3 + 2 \log\!\left( 1 + \tfrac{c_z^2}{8\lambda} \right) }.
\]
\end{lemma}

\begin{proof}{Proof of Lemma~\ref{lem:param-error}}
Define the regularized negative log-likelihood:
\[
  \mathcal{L}_t(\theta)
  := -\sum_{i=1}^{\widetilde t} w_i \left[
    y_i \log \sigma(z_i^\top \theta)
    + (1 - y_i) \log (1 - \sigma(z_i^\top \theta))
  \right] + \frac{\lambda}{2} \|\theta\|_2^2.
\]
Since $\nabla \mathcal{L}_t(\hat\theta_t) = 0$, the mean-value expansion at $\theta$ yields:
\[
  H_{\widetilde t}  (\hat\theta_t - \theta) = \xi_t,
  \quad
  H_{\widetilde t}  := \nabla^2 \mathcal{L}_t(\tilde\theta),\quad
  \xi_t := \sum_{i=1}^{\widetilde t}
           w_i (y_i - \sigma(z_i^\top \theta)) z_i,
\]
for some $\tilde\theta$ on the line segment between $\theta$ and $\hat\theta_t$.
For all $|z_i^\top \tilde\theta| \le B$, the Hessian satisfies:
\[
  c_\mu A_{\widetilde t}  \;\preceq\; H_{\widetilde t}  \;\preceq\; k_\mu A_{\widetilde t} .
\]
Lemma~\ref{lem:selfnorm} implies $\|\xi_t\|_{A_t^{-1}} \le \beta_t$. Since $H_{\widetilde t} ^{-1} \preceq (1/c_\mu) A_{\widetilde t} ^{-1}$, we have:
\[
  \|H_{\widetilde t} ^{-1/2} \xi_t\|_2
  \le \frac{\beta_{\widetilde t} }{\sqrt{c_\mu}}.
\]
Because $H_{\widetilde t}  (\hat\theta_t - \theta) = \xi_t$, we have:
\[
  \|H_{\widetilde t} ^{1/2}(\hat\theta_t - \theta)\|_2
  = \|H_{\widetilde t} ^{-1/2} \xi_t\|_2
  \le \frac{\beta_{\widetilde t} }{\sqrt{c_\mu}},
\]
and since $H_t \succeq c_\mu A_t$,
\[
  \|\hat\theta_t - \theta\|_{A_{\widetilde t} }
  \le \frac{\beta_{\widetilde t} }{c_\mu}.
\]
If $\lambda_{\min}(A_{\widetilde t} ) \ge \lambda + \tfrac{1}{2} w_{\min} p_{\min} \widetilde t$, then:
\[
  \|\hat\theta_t - \theta\|_2
  \le
  \frac{ \|\hat\theta_t - \theta\|_{A_{\widetilde t} } }
       { \sqrt{ \lambda_{\min}(A_{\widetilde t} ) } }
  \le
  \frac{ \beta_{\widetilde t}  }
       { c_\mu \sqrt{ \lambda + \tfrac{1}{2} w_{\min} p_{\min} \widetilde t } }.
\]
\end{proof}
\bigskip
\begin{lemma}[Cumulative Greedy Regret]
\label{lem:cumgreedy-regret}
Assume \textnormal{(A1)}–\textnormal{(A3)}.
Suppose the parameter estimate satisfies the high-probability bound
from Lemma~\ref{lem:param-error}:
\[
  \bigl\|\hat\theta_t - \theta \bigr\|_2
  \;\le\;
  \frac{\beta_{\widetilde t} }{c_\mu\sqrt{\lambda + \tfrac12 w_{\min}p_{\min} \widetilde t}},
  \quad\text{for all } t \ge 1,
\]
where
\[
  \beta_{\widetilde t}  := \kappa\,e_{\max}c_z
  \sqrt{
    2(d+K)\log {\widetilde t}  \cdot \log\!\left( \tfrac{d+K}{\delta} \right)
  }.
\]

Then, for any $T \ge 1$, with probability at least $1 - \delta$,
\[
  \sum_{t=T_0+1}^{T} \sum_{n \notin \mathcal T_t}
    \bigl(
      V_1^{\pi^*}(x_{t,n}) - V_1^{\pi^{grd}}(x_{t,n})
    \bigr)
  \;\le\;
  \frac{4e_{\max} \,c_z}
       {c_\mu \hat\sigma_{\min} \sqrt{\tfrac12 w_{\min}p_{\min}}}
  \cdot \sqrt{T} \cdot  \beta_{\widetilde T}.
\]
\end{lemma}

\begin{proof}{Proof of Lemma~\ref{lem:cumgreedy-regret}}
For each episode $t$, only the $N - M_t$ non-exploration contexts incur regret:
\[
  \sum_{n \notin \mathcal T_t}
    \bigl|
      V_1^{\pi^*}(x_{t,n}) - V_1^{\pi^{grd}}(x_{t,n})
    \bigr|
  \;\le\;
  \frac{2e_{\max}\,c_z}{\hat\sigma_{\min}}
  \sqrt{N - M_t} \cdot \bigl\| \hat\theta_t - \theta \bigr\|_2.
\]

Using the parameter error bound and the inequality
\(
  \lambda_{\min}(A_{\widetilde t} )
  \ge \lambda + \tfrac12 w_{\min}p_{\min}\widetilde t
  \ge \tfrac12 w_{\min}p_{\min} Nt,
\)
we get:
\[
  \sqrt{N - M_t} \cdot \|\hat\theta_t - \theta\|_2
  \;\le\;
  \frac{\beta_{\widetilde t} }{c_\mu\sqrt{\lambda + \tfrac12 w_{\min}p_{\min} \widetilde t}}
  \cdot \frac{\sqrt{N - M_t}}{1}
  \;\le\;
  \frac{\beta_{\widetilde t} }{c_\mu \sqrt{\tfrac12 w_{\min}p_{\min}}}
  \cdot \frac{\sqrt{N - M_t}}{Nt}
    \;\le\;
  \frac{\beta_{\widetilde T} }{c_\mu \sqrt{\tfrac12 w_{\min}p_{\min}}}
  \cdot \frac{1}{t}
\]

Summing over all $t$ and using
\(
  \sum_{t=T_0}^{T} \tfrac{1}{\sqrt{t}}
  \le 2\sqrt{T},
\)
we obtain:
\[
  \sum_{t=T_0+1}^{T} \sum_{n \notin \mathcal T_t}
    \bigl(
      V_1^{\pi^*} - V_1^{\pi^{grd}}
    \bigr)
  \;\le\;
  \frac{4e_{\max}\,c_z}
       {c_\mu \hat\sigma_{\min} \sqrt{\tfrac12 w_{\min}p_{\min}}}
  \cdot \sqrt{T} \cdot  \beta_{\widetilde T}.
\]

Since $\beta_t$ is non-decreasing in $t$, we conclude
$\max_{1 \le t \le {\widetilde T}} \beta_t = \beta_{\widetilde T}$.
\end{proof}
\bigskip
\begin{lemma}[Exploration regret bound]
\label{lem:exploration-tight}
Let $M_t := |E_t|$ be the number of exploratory contexts in episode~$t$, and define the exploration regret:
\[
R_{\exp}(T)
\;:=\;
\sum_{t=T_0+1}^{T}\sum_{n\in E_t}
  \bigl(V_1^{{\pi^*}}(x_{t,n}) - V_1^{\pi_{\mathrm{ucb}}}(x_{t,n})\bigr).
\]
Assume the high-probability events of Lemmas~\ref{lem:selfnorm}, \ref{lem:HfreeQbound}, and~\ref{lem:max_radius_uniform} all hold. Then
\[
R_{\exp}(T)
\;\le\;
\frac{4\,e_{\max}\,k_{\mu}}{c_{\mu}\,\hat\sigma_{\min}}\,
\beta_{\widetilde T}\,
\sqrt{\sum_{t=1}^{T} M_t}\,
\sqrt{
  \frac{4\,c_{z}^{2}}{w_{\min}p_{\min}}\,
  \log\!\left(
    1 + \frac{w_{\min}p_{\min}}{2\lambda} \cdot \widetilde T
  \right)
}\]
\[
+ 2\,\beta_{\widetilde T}\,
\sqrt{H \sum_{t=1}^{T} M_t}\,
\sqrt{
  \frac{4\,c_{z}^{2}}{w_{\min}p_{\min}}\,
  \log\!\left(
    1 + \frac{w_{\min}p_{\min}}{2\lambda} \cdot \widetilde T
  \right)
}.
\]
\end{lemma}

\begin{proof}{Proof of Lemma~\ref{lem:exploration-tight}}
Let $R_{\exp}(T)$ denote the exploration regret. For each exploratory context $x_{t,n}$ and step $h \in [H]$, we decompose the value difference as:
\[
\bigl|
  V_h^{\pi^*}(x_{t,n,h})
  - V_h^{\pi^{\mathrm{ucb}}}(x_{t,n,h})
\bigr|
\;=\;
\bigl|
  Q_h^{*}(x_{t,n,h},k^{*})
  - Q_h^{*}(x_{t,n,h},k^{\mathrm{ucb}})
\bigr|
\]
\[
\;\le\;
\underbrace{
  \bigl| Q_h^{*}(x_{t,n,h},k^{*}) - \hat Q_h(x_{t,n,h},k^{*}) \bigr|
}_{\text{(value deviation of }k^{*}\text{)}}
+
\underbrace{
  \bigl| \hat Q_h(x_{t,n,h},k^{*}) - \hat Q_h(x_{t,n,h},k^{\mathrm{ucb}}) \bigr|
}_{\text{(UCB gap)}}
+
\underbrace{
  \bigl| \hat Q_h(x_{t,n,h},k^{\mathrm{ucb}}) - Q_h^{*}(x_{t,n,h},k^{\mathrm{ucb}}) \bigr|
}_{\text{(value deviation of }k^{\mathrm{ucb}}\text{)}}.
\]

\medskip

\noindent\textbf{Value deviation terms.}  
Lemma~\ref{lem:Q-value-stability} implies that each of the two symmetric value-deviation terms is bounded, and together:
\[
\sum_{t=T_0+1}^{T}\sum_{n\in E_t}
\bigl|\,
Q_{1}^{*}(x_{t,n},k)
- \hat Q_{1}(x_{t,n},k)
\bigr|
\;\le\;
\frac{2\,e_{\max}\,k_{\mu}}
     {c_{\mu}\,\hat\sigma_{\min}}\;
\beta_{\widetilde T}
\sqrt{\sum_{t=1}^{T} M_t}
\sqrt{
  \frac{4\,c_z^2}{w_{\min}p_{\min}}\,
  \log\!\left(1 + \frac{w_{\min}p_{\min}}{2\lambda}\,\widetilde T\right)
}.
\]
Multiplying by 2 (due to the two value-deviation terms) yields the first term in the claimed bound.

\medskip

\noindent\textbf{UCB gap term.}  
By Lemma~\ref{lem:ucbbp_gap} and applying Lemma~\ref{lem:cum_ucb_gap} to all steps $h$ across exploration contexts:
\[
\sum_{t=T_0+1}^{T}\sum_{n\in E_t}\sum_{h=1}^{H}
\bigl| \hat Q(x_{t,n},k) - \hat Q(x_{t,n},k^{\mathrm{ucb}}) \bigr|
\;\le\;
2\,\beta_{\widetilde T}
\sqrt{H \sum_{t=1}^{T} M_t}
\sqrt{
  \frac{4\,c_z^2}{w_{\min}p_{\min}}\,
  \log\!\left(1 + \frac{w_{\min}p_{\min}}{2\lambda}\,\widetilde T\right)
}.
\]

\medskip

\noindent Combining all terms completes the proof.
\end{proof}

\bigskip

\section{Total Regret of UCBBP}\label{append:ucbbp proof}

Please note that we present all supporting lemmas in Appendix B.

\label{AppedixproofTh1}
\begin{proof}{Proof of Theorem ~\ref{thm:UCBBP}}
For $t\le T_{0}$ each of the $N$ users can incur at most $e_{\max}$
regret, yielding $N T_{0} e_{\max}$.
\[
  \bigl|
    V_h^{\pi^*}(x_{t,n,h},)
    -V_h^{\pi^{ucb}}(x_{t,n,h})
  \bigr|
  \;=\;
  \bigl|
    Q_h^{*}(x_{t,n,h},k^{*})
    -Q_h^{*}(x_{t,n,h},k^{\mathrm{ucb}})
  \bigr|\]
  \[
  \;\le\;
  \underbrace{\bigl|
      Q_h^{*}(x_{t,n,h},k^{*})
      -\hat Q_h(x_{t,n,h},k^{*})
    \bigr|}_{\text{(value deviation of }k^{*}\text{)}}\;+\;
  \underbrace{\bigl|
      \hat Q_h(x_{t,n,h},k^{*})
      -\hat Q_h(x_{t,n,h},k^{\mathrm{ucb}})
    \bigr|}_{\text{(UCB gap)}}\;+\;
  \underbrace{\bigl|
      \hat Q_h(x_{t,n,h},k^{\mathrm{ucb}})
      -Q_h^{*}(x_{t,n,h},k^{\mathrm{ucb}})
    \bigr|}_{\text{(value deviation of }k^{\mathrm{ucb}}\text{)}}.
\]
\textbf{Value-deviation terms.}  
Lemma \ref{lem:Q-value-stability} (applied over
$t=T_{0}+1,\dots,T$ and $n=1,\dots,N$) gives
\[
\sum_{t=T_{0}+1}^{T}\sum_{n=1}^{N}
  \bigl|\,
     Q_{1}^{*}(x_{t,n},k)
     -\hat Q_{1}(x_{t,n},k)
  \bigr|
\;\le\;
\frac{2\,e_{\max}\,k_{\mu}}
     {c_{\mu}\,\hat\sigma_{\min}}\;
\beta_{\widetilde T}\,
\sqrt{N\,T}\,
\sqrt{\frac{4\,c_{z}^{2}}{w_{\min}p_{\min}}\,
      \log\!\Bigl(
        1+\frac{w_{\min}p_{\min}}{2\lambda}\,\widetilde T
      \Bigr)}.
\]
Because this bound applies to each of the two symmetric
value-deviation terms, we multiply by~$2$ to obtain the first large
square-root term in the theorem.

\medskip
\textbf{UCB-gap term.}  
For the flattened indices $i=i_{0},\dots,\widetilde T$
corresponding to $t>T_{0}$, Lemma \ref{lem:cum_ucb_gap} yields
\[
\sum_{i=i_{0}}^{\widetilde T}\!
  \bigl|
     \hat Q_{h(i)}(x_i,k_i^{*})
     -\hat Q_{h(i)}(x_i,k_i^{\mathrm{ucb}})
  \bigr|
\;\le\;
2\,\beta_{\widetilde T}\,
\sqrt{\widetilde T}\,
\sqrt{\frac{4\,c_{z}^{2}}{w_{\min}p_{\min}}\,
      \log\!\Bigl(
        1+\frac{w_{\min}p_{\min}}{2\lambda}\,\widetilde T
      \Bigr)}.
\]

\medskip
\textbf{Combine.}  
Adding the warm-up term and the two post-warm-up bounds completes the
proof.
\end{proof}
\bigskip
\section{Total Regret of AUCBBP}\label{append:aucbbp proof}
\label{Appendixproofth2}

Please note that we present all supporting lemmas in Appendix B.

\begin{proof}{Proof of Theorem ~\ref{thm:AUCBBP}}
We decompose the total regret into three parts:

\textbf{Warm-up regret.}
The warm-up phase lasts \(T_0\) episodes, each with \(N\) users and regret at most \(e_{\max}\), giving
\[
R_{\mathrm{warm}} = T_0 N e_{\max}.
\]

\textbf{Exploration regret.}
By Lemma~\ref{lem:exploration-tight}, and letting \(\widetilde T := NT\), we have
\begin{align*}
R_{\mathrm{exp}}(T)
\;\le\;&\,
\frac{4\,e_{\max}\,k_{\mu}}{c_{\mu}\,\hat\sigma_{\min}}\,
\beta_{\widetilde T}\,
\sqrt{\sum_{t=1}^{T} M_t}\,
\sqrt{
  \frac{4\,c_{z}^{2}}{w_{\min}p_{\min}}\,
  \log\!\left(
    1 + \frac{w_{\min}p_{\min}}{2\lambda} \cdot \widetilde T
  \right)
} \\
&+ 2\,\beta_{\widetilde T}\,
\sqrt{H \sum_{t=1}^{T} M_t}\,
\sqrt{
  \frac{4\,c_{z}^{2}}{w_{\min}p_{\min}}\,
  \log\!\left(
    1 + \frac{w_{\min}p_{\min}}{2\lambda} \cdot \widetilde T
  \right)
}.
\end{align*}

\textbf{Greedy regret.}
From Lemma~\ref{lem:greedy-regret}, we have
\[
R_{\mathrm{grd}}(T)
\;\le\;
\frac{4e_{\max} \,c_z}
     {c_\mu \hat\sigma_{\min} \sqrt{\tfrac12 w_{\min}p_{\min}}}
\cdot \sqrt{T} \cdot  \beta_{\widetilde T}.
\]

\textbf{Total bound.}
Summing all three parts yields the stated result. Since
\[
\sum_{t=1}^T M_t = \mathcal{O}(N \log^2 T),
\quad \text{and} \quad
\beta_{\widetilde T} = \widetilde{\mathcal{O}}\left(\sqrt{(d+K)}\right),
\]
In asymptotic notation, this yields:
\[
\operatorname{Regret}(NT)
\;=\;
\widetilde{\mathcal{O}}\left(
  \beta_{\widetilde T}
  \cdot
  \sqrt{(d+K)(H N\log^2 T + T)}
\right)
\;=\;
\widetilde{\mathcal{O}}\left(
  \sqrt{(d+K)(H N + T)}
\right).
\]
\end{proof}
\bigskip

\end{APPENDICES}

\end{document}